%% file: kdd2026.tex
\documentclass[sigconf]{acmart}
\usepackage{tikz}
\usepackage{pgfplots}
\pgfplotsset{compat=1.18}
\usepackage{float}
\usepackage{amsthm}
\usepackage{amsfonts}
\usepackage{centernot}
\usepackage{graphicx}
\usepackage{siunitx}
\usepackage{multirow}
\usepackage{mathtools}
\usepackage{soul}
\usepackage{cleveref}
\usepackage{subcaption}
\usepackage{pifont}
\usepackage{algorithm}
\usepackage{algpseudocode}
\usepackage{booktabs}  % For toprule, midrule, bottomrule
\usepackage{float}     % For [H] option (if using Solution 2)
\usepackage{caption} 
\usepackage{enumitem}
\usepackage{thmtools} 
\usepackage{thm-restate}

\declaretheorem[
    name=Theorem,
    numberwithin=section,
    refname={Theorem,Theorems},
    Refname={Theorem,Theorems}
]{theorem}

\declaretheorem[
    name=Definition,
    style=definition,
    sibling=theorem
]{definition}

\usepackage{xcolor}

\begin{document}

\title{Invariant-Stratified Propagation for Expressive Graph Neural Networks}

\author{Asela Hevapathige}
%\authornotemark[1]
\email{asela.hevapathige@unimelb.edu.au}
\affiliation{%
  \institution{Department of Mechanical Engineering \\ University of Melbourne}
  \city{Melbourne}
  \country{Australia}
}

\author{Ahad N. Zehmakan}
\email{ahadn.zehmakan@anu.edu.au}
\affiliation{%
  \institution{School of Computing \\ Australian National University}
  \city{Canberra}
  \country{Australia}
}

\author{Asiri Wijesinghe}
\email{asiriwijesinghe.wijesinghe@data61.csiro.au}
\affiliation{%
  \institution{Data61 \\ CSIRO}
  \city{Canberra}
  \country{Australia}
}

\author{Saman Halgamuge}
%\authornotemark[1]
\email{saman@unimelb.edu.au}
\affiliation{%
  \institution{Department of Mechanical Engineering \\ University of Melbourne}
  \city{Melbourne}
  \country{Australia}
}

\begin{abstract}
Graph Neural Networks (GNNs) face fundamental limitations in expressivity and capturing structural heterogeneity. Standard message-passing architectures are constrained by the 1-dimensional Weisfeiler-Leman (1-WL) test, unable to distinguish graphs beyond degree sequences, and aggregate information uniformly from neighbors, failing to capture how nodes occupy different structural positions within higher-order patterns. While methods exist to achieve higher expressivity, they incur prohibitive computational costs and lack unified frameworks for flexibly encoding diverse structural properties. To address these limitations, we introduce Invariant-Stratified Propagation (ISP), a framework comprising both a novel WL variant (ISP-WL) and its efficient neural network implementation (ISP-GNN). ISP stratifies nodes according to graph invariants, processing them in hierarchical strata that reveal structural distinctions invisible to 1-WL. Through hierarchical structural heterogeneity encoding, ISP quantifies differences in nodes' structural positions within higher-order patterns, distinguishing interactions where participants occupy vastly different roles from those with uniform participation. We provide formal theoretical analysis establishing enhanced expressivity beyond 1-WL, convergence guarantees, and inherent resistance to oversmoothing. Extensive experiments across graph classification, node classification, and influence estimation demonstrate consistent improvements over both standard architectures and state-of-the-art expressive baselines. %Our source code is available at: \href{https://anonymous.4open.science/r/ISP-GNN-C547/}{\textcolor{blue}{\nolinkurl{https://anonymous.4open.science/r/ISP-GNN-C547/}}}.
\end{abstract}




\maketitle

\input{sections/introduction}
\input{sections/related_work}

\input{sections/concepts}
\input{sections/methodology}

\input{sections/theoretical_analysis}
\input{sections/experiments}

\input{sections/conclusion}

\bibliography{references}
\bibliographystyle{ACM-Reference-Format}

\clearpage
\balance

\input{sections/appendix}

\end{document}

%% file: sections/introduction.tex
\section{Introduction}
Graph Neural Networks (GNNs) have emerged as a powerful tool for learning on graph-structured data, with applications spanning molecular property prediction \cite{wieder2020compact,wang2023graph}, social network analysis \cite{ben2024enhancing,hevapathige2025deepsn}, and knowledge reasoning \cite{zhang2022knowledge,liang2024survey}. Many GNNs are built on message-passing mechanisms, where nodes iteratively aggregate information from neighbors to refine their representations, allowing them to model relational structure \cite{gilmer2017neural}.
However, several fundamental limitations constrain their ability to represent complex graph structures:

\begin{itemize}
\item \textbf{Computational-Expressivity Trade-off.} Standard GNNs are fundamentally restricted by the 1-dimensional Weisfeiler-Leman (1-WL) test \cite{weisfeiler1968reduction,xu2018powerful}, distinguishing graphs only by neighborhood degree sequences. Consequently, nodes with identical local structural patterns receive indistinguishable representations regardless of their global roles \cite{wijesinghe2022new,you2021identity}. This creates a critical trade-off: methods achieving higher expressivity, such as $k$-WL variants \cite{morris2019weisfeiler,zhao2022practical,azizianexpressive} and subgraph enumeration approaches \cite{bevilacqua2021equivariant,cotta2021reconstruction,zhao2022stars} incur prohibitive computational costs that scale poorly to large graphs. While structural information injection techniques \cite{wijesinghe2022new,bouritsas2022improving,xu2024union,hevapathige2025curvature} offer computational efficiency, they typically incorporate fixed structural encodings predetermined by architectural design. This rigidity prevents models from adaptively selecting and combining diverse structural properties such as geometric features, topological signatures, and higher-order patterns based on task-specific requirements and data characteristics.

\item \textbf{Inability to Capture Structural Heterogeneity in Higher-Order Patterns.} Compounding the expressivity challenge above, standard GNNs aggregate information uniformly from all neighbors, treating higher-order structures identically, regardless of the structural positions of participating nodes \cite{qian2022ordered,wanigatunga2025uncovering}. They cannot distinguish higher-order interactions where nodes occupy vastly different structural positions. This structural heterogeneity within higher-order patterns is critical across diverse applications: from network analysis to molecular property prediction, where the relative structural positions of entities within local patterns strongly influence properties and behavior.

\item \textbf{Static, Task-Agnostic Structural Encoding.} Current GNN methods encode structural information primarily through isolated mechanisms. For instance, positional encodings use spectral or random walk features \cite{dwivedi2021generalization,srinivasanequivalence}, subgraph methods enumerate predetermined patterns \cite{zhao2022stars,bevilacqua2021equivariant}, and rewiring techniques modify topology using geometric tools such as curvature \cite{toppingunderstanding,karhadkarfosr}. These methods lack a unified framework for flexibly combining multiple graph invariants, connecting global structural orderings with local higher-order interactions, and adapting structural biases to task requirements through learning. This fragmentation prevents models from discovering which structural properties and their interplay are most informative for specific prediction tasks.
\end{itemize}

\vspace{-15mm}

In this paper, we introduce Invariant-Stratified Propagation (ISP), a unified framework comprising both a novel WL variant (ISP-WL) and its efficient neural implementation (ISP-GNN) that systematically addresses the aforementioned limitations. The key insight of our work is orderings induced by graph invariants, which we term "stratification" based on ordinal rankings of invariant values.  It naturally reveals structural heterogeneity patterns within higher-
order structures that are invisible to uniform aggregation schemes. ISP-WL processes nodes according to these invariant-induced orderings, enabling discrimination of structurally distinct nodes that appear identical under standard message-passing. ISP-GNN implements these principles through a differentiable mechanism that learns to combine multiple structural properties. Our main contributions are:

\begin{itemize}[topsep=0pt, itemsep=0pt, parsep=0pt]
    \item \textbf{Invariant-Stratified Propagation:} We propose a framework that enhances standard message passing with invariant-based inductive biases to overcome the degree sequence barrier. By stratifying nodes based on graph invariants, we process them in ways that reveal structural differences otherwise missed by uniform aggregation, achieving greater expressivity while remaining computationally efficient.
    
    \item \textbf{Structural Heterogeneity Encoding:} We introduce a method for encoding hierarchical structural differences that captures the varying roles of nodes in higher-order patterns. This approach distinguishes interactions based on the structural positions of nodes, contrasting standard uniform aggregation that overlooks these differences.

    \item \textbf{Unified Framework for Structural Encoding:} Our framework integrates diverse structural information, offering support for both predefined and learnable invariants. This unifies fragmented existing methods, allowing for flexible composition of multiple graph invariants within a single framework.
\end{itemize}

Beyond the core methodological contributions, we provide formal theoretical analysis establishing the ISP Framework's enhanced expressivity beyond 1-WL, convergence guarantees, computational efficiency matching standard message-passing on sparse graphs, and inherent resistance to oversmoothing through depth-invariant structural embeddings. Extensive experiments across diverse graph learning scenarios, including graph classification, node classification, and dynamic influence estimation tasks, demonstrate consistent and substantial improvements of our approach over both standard message-passing architectures and expressive baselines.

\vspace{-3mm}

%% file: sections/related_work.tex
\section{Related Work}

\subsection{GNNs Beyond 1-WL Expressive Power}

Standard message-passing GNNs are fundamentally limited by the 1-dimensional Weisfeiler-Leman test, restricting their ability to distinguish graphs beyond degree sequences \citep{xu2018powerful,weisfeiler1968reduction}. Higher-order WL methods \citep{azizian2021expressive,morris2019weisfeiler,zhao2022practical} extend to tuples of nodes for greater expressivity but incur prohibitive computational costs that scale poorly to large graphs. Subgraph-based methods \citep{bouritsas2022improving,bevilacqua2021equivariant,liuempowering} enhance expressivity through subgraph enumeration and aggregation but suffer from substantial computational overhead and lack adaptive mechanisms to weight structural patterns based on task requirements. Structural encoding approaches \citep{you2021identity,wijesinghe2022new,hevapathige2025curvature,xu2024union} inject heterogeneous node identifiers or structural features to distinguish otherwise identical nodes. However, these methods typically use fixed, predefined structural properties that may not align with task-specific requirements and lack unified frameworks for flexibly combining multiple graph invariants. 

ISP-WL alleviates the computational-expressivity trade-off by achieving expressivity beyond standard message-passing while maintaining computational efficiency on sparse graphs through invariant-based stratification. Our framework provides both efficient triangle-based aggregation with a unified mechanism for adaptively combining multiple structural invariants through end-to-end training.

\vspace{-4mm}

\subsection{Ordering and Hierarchy in Graph Learning}

Recent work has explored various approaches to organize and structure graph neural network computations beyond uniform aggregation. Hierarchical methods \citep{ying2018hierarchical,bianchi2020spectral,gao2019graph,wu2024hierarchy,zhang2025graph} employ coarsening and pooling operations to learn multi-scale graph representations, but operate primarily on graph-level tasks without providing node-level structural stratification during message-passing. Position-aware GNNs \cite{you2019position,nishad2020graphreach,galron2025understanding} incorporate spectral and random walk-based positional encodings to distinguish node positions, while direction-aware methods like DirGNN \cite{rossi2024edge} and ordered aggregation approaches like OSAN \citep{qian2022ordered} and GOAT \cite{chatzianastasis2023graph} explicitly model edge directionality and neighbor ordering. However, these methods treat structural properties as static input features or apply local ordering heuristics without principled frameworks for determining orderings based on global graph structure, and they do not explicitly model how nodes occupy different structural positions within higher-order patterns.

ISP-GNN utilizes invariant-induced stratification to create globally consistent orderings for processing nodes based on their structural importance, measured by diverse graph invariants. This hierarchical approach highlights structural heterogeneity, as nodes at different invariant levels hold distinct structural roles. By explicitly quantifying differences within triangles through invariant-based gap measures, the model can differentiate between nodes with varying structural importance. This integration of structural hierarchy into the propagation mechanism allows the model to adaptively learn which hierarchical patterns are relevant for specific tasks.

%% file: sections/methodology.tex
\vspace{-3mm}

\section{ISP Framework}

We consider an undirected graph \( G = (V, E) \) with node set \( V \) and edge set \( E \). For a node \( v \in V \), its neighborhood is \( N(v) = \{u \in V : (v,u) \in E\} \) and its degree is \( d_v = |N(v)| \). The triangles incident to \( v \) are given by \( T(v) = \{\{u,w\} : u,w \in N(v), (u,w) \in E\} \). We denote multisets with \(\{\!\!\{\cdot\}\!\!\}\) and use \(\text{Hash}(\cdot)\) for injective hash functions. A graph invariant is a function \( \mu: V \to \mathbb{R} \) that holds under graph isomorphisms, meaning \( \mu(\pi(v)) = \mu(v) \) for any isomorphism \( \pi: V \to V \).
For detailed background concepts related to our study, please refer to the Appendix section \ref{sec:bac}.

\subsection{ISP-WL Algorithm}

For invariants, we apply \textbf{global ranking}: $\phi(v) = \text{rank}(\{\mu(u) : u \in \bigcup_i V_i\}, \mu(v))$, mapping values to consecutive integers $1, 2, \ldots, L$ where $L = |\{\mu(v) : v \in \bigcup_i V_i\}|$ is the number of distinct invariant values across all graphs in the comparison set. This ensures consistent representations and defines $L$ distinct layers.

\subsubsection{Hierarchical Structural Heterogeneity Encoding}

The hierarchical processing order determined by $\phi$ creates a natural stratification that reveals structural heterogeneity within higher-order structures.

\begin{definition}[Hierarchical Structural Heterogeneity Encoding]
For a higher-order interaction $(v, u, w)$ where $v$ is the center node and $u, w \in N(v)$ with $(u,w) \in E$:
\begin{equation}
\delta(v, u, w) = (\delta_1, \delta_2, \delta_3)
\end{equation}
where:
\begin{align}
\delta_1 &= \max(\phi(v) - \phi(u), \phi(v) - \phi(w)) \quad \text{(largest gap)} \\
\delta_2 &= \min(\phi(v) - \phi(u), \phi(v) - \phi(w)) \quad \text{(smallest gap)} \\
\delta_3 &= |\phi(u) - \phi(w)| \quad \text{(inter-neighbor gap)}
\end{align}
\end{definition}

These three values provide a complete, order-independent characterization of the gap structure in a triangle, serving as a minimal sufficient statistic for structural heterogeneity. The design rationale of this encoding and its natural extension to higher-order structures beyond triangles is provided in Appendix Section \ref{sec:hie_gap}.

The following proposition formalizes that the encoding treats neighbor pairs symmetrically, which is crucial for correctness and ensures consistent representations of graphs regardless of the order in which neighbors are enumerated.

\begin{restatable}[Permutation Invariance]{proposition}{permutationinv} \label{prop:}\label{prop:permutationinv}
$\delta(v, u, w) = \delta(v, w, u)$ for all $u, w \in N(v)$.
\end{restatable}

\subsubsection{Hierarchical Color Refinement}

The algorithm iteratively refines two color streams until convergence. The WL stream follows standard 1-WL refinement, while the ISP stream processes nodes in $L$ layers corresponding to the $L$ distinct invariant values. At iteration $\ell \in \{1, 2, \ldots, L\}$, we process all nodes $v$ with $\phi(v) = \ell$, constructing their ISP colors. When processing node $v$ with $\phi(v) = \ell$, we construct:
\begin{equation}
M^{\text{struct}}_v = \{(c_{\text{ISP}}(u), c_{\text{ISP}}(w), \delta(v, u, w)) : (u, w) \in T(v)\}
\end{equation}
where $c_{\text{ISP}}(u), c_{\text{ISP}}(w) \in \mathbb{N} \cup \{\bot\}$ are ISP colors ($\bot$ represents uncolored nodes). The ISP color is assigned as:
\begin{equation}
c_{\text{ISP}}(v) = \begin{cases}
\text{Hash}(\phi(v), M^{\text{struct}}_v) & \text{if } M^{\text{struct}}_v \neq \{\} \\
\phi(v) & \text{otherwise}
\end{cases}
\end{equation}

A key property of ISP-WL is that each node's ISP color is computed exactly once and persists thereafter, eliminating redundant computation and ensuring the ISP stream completes in exactly $L$ iterations. This is formalized below.

\begin{restatable}[Single Assignment Property]{proposition}{singleassign}
For any node $v \in V$, $c_{\text{ISP}}(v)$ is assigned exactly once at iteration $\ell^* = \phi(v)$ and remains unchanged thereafter.
\end{restatable}

\begin{algorithm}[t]
\caption{ISP-WL Algorithm}
\label{alg:isp-wl}
\begin{algorithmic}[1]
\State \textbf{Input:} Graph $G = (V, E)$, globally-ranked invariant $\phi : V \to \{1, \ldots, L\}$
\State \textbf{Output:} Coloring $c : V \to \mathbb{N} \times \mathbb{N}$
\State Initialize: $c_{WL}(v) \gets 0$, $c_{ISP}(v) \gets \bot$ for all $v \in V$
\State $\ell \gets 1$
\Repeat
    \State \textcolor{blue}{\textit{// WL Refinement}}
    \For{$v \in V$}
        \State $m_v \gets \{\!(c_{WL}(u), c_{ISP}(u)) : u \in N(v)\!\}$ 
        \State $c'_{WL}(v) \gets \mathit{Hash}((c_{WL}(v), c_{ISP}(v)), m_v)$
    \EndFor
    \State $c_{WL}^{\text{prev}} \gets c_{WL}$
    \State $c_{WL} \gets c'_{WL}$
    \State \textcolor{blue}{\textit{// ISP Assignment}}
    \If{$\ell \leq L$}
        \State $V_\ell \gets \{v \in V : \phi(v) = \ell\}$
        \For{$v \in V_\ell$}
            \State $M^{\text{struct}}_v \gets \{\!(c_{ISP}(u), c_{ISP}(w), \delta(v, u, w)) : (u, w) \in T(v)\!\}$
            \If{$M^{\text{struct}}_v \neq \{\!\}$}
                \State $c_{ISP}(v) \gets \mathit{Hash}(\phi(v), M^{\text{struct}}_v)$
            \Else
                \State $c_{ISP}(v) \gets \phi(v)$
            \EndIf
        \EndFor
    \EndIf
    \State $\ell \gets \ell + 1$
\Until{($c_{WL} = c_{WL}^{\text{prev}}$) AND ($\ell > L$)} \textcolor{blue}{\textit{// Both streams converged}}
\State \textbf{return} $c(v) = \mathit{Hash}(c_{WL}(v), c_{ISP}(v))$ for all $v \in V$
\end{algorithmic}
\end{algorithm}

The algorithm terminates when both color streams stabilize, which occurs in at most $\max(K_{\text{WL}}, L)$ iterations where $K_{\text{WL}}$ is the standard 1-WL convergence time. The pseudocode of ISP-WL is provided in Algorithm \ref{alg:isp-wl}.

\begin{figure*}[htbp]
    \centering
    \includegraphics[width=0.75\textwidth, height=0.3\textheight, keepaspectratio]{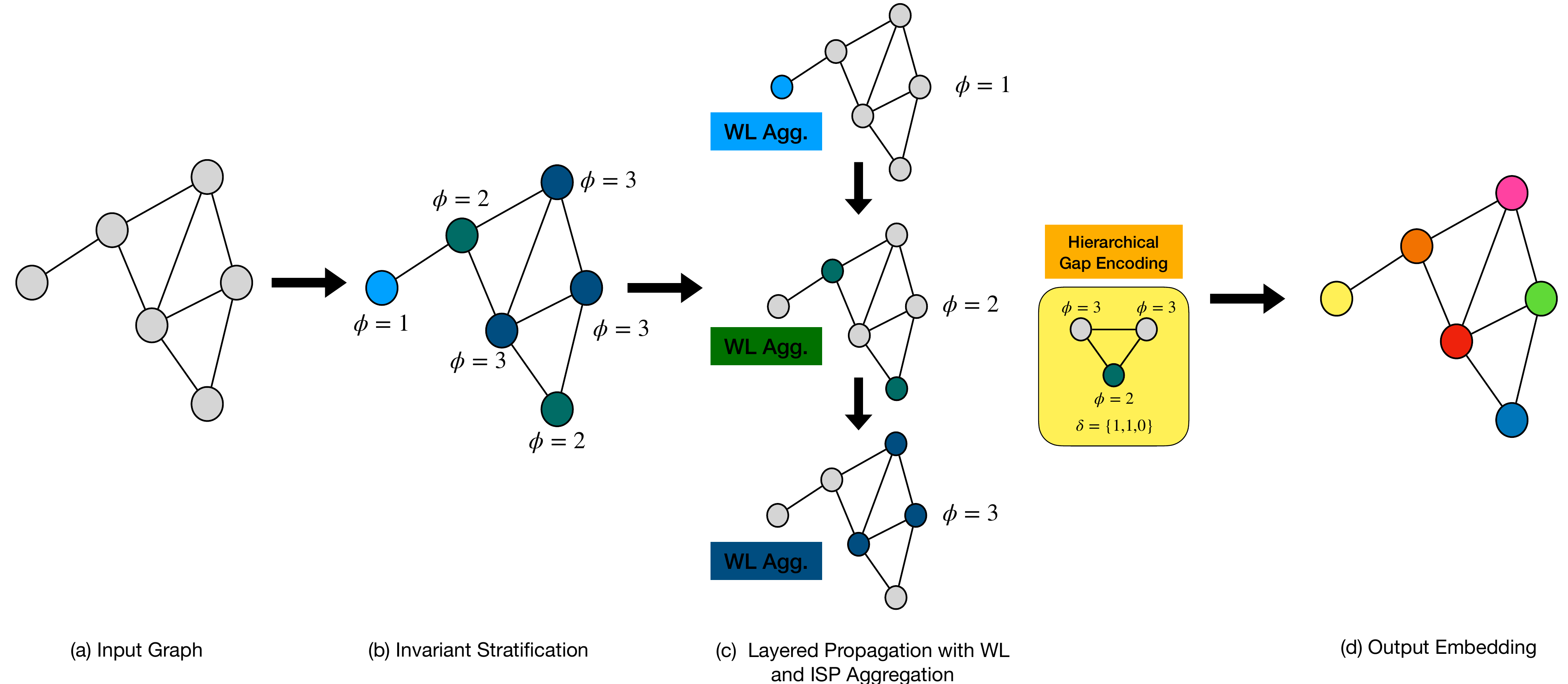}
    \caption{ISP-GNN architecture overview. (a) Input graph. (b) ISP coloring based on invariant values. (c) Layered propagation with WL and ISP aggregation at each layer. Triangle detail shows gap encoding for nodes with higher-layer neighbors. (d) Output embeddings for downstream tasks.}
    \label{fig:isp_gnn_architecture}
\end{figure*}

\subsection{ISP-GNN Architecture}

We translate the discrete ISP-WL algorithm into a continuous, differentiable neural network using learned attention mechanisms and soft gating. ISP-GNN maintains a dual-stream design with learned hierarchical attention that weights higher-order interactions based on structural heterogeneity. At layer $k$, nodes maintain $h^{(k),\text{WL}}_v$ (WL features) and $h^{(k),\text{ISP}}_v$ (ISP features).

\vspace{-3mm}
\subsubsection{Neural Hierarchical Structural Heterogeneity Encoding}

For each higher-order interaction $(v, u, w)$, we compute hierarchical heterogeneity features as the continuous analog of $\delta$:
\begin{equation}
d_{vuw} = \begin{bmatrix}
\max(\phi(v) - \phi(u), \phi(v) - \phi(w)) \\
\min(\phi(v) - \phi(u), \phi(v) - \phi(w)) \\
|\phi(u) - \phi(w)|
\end{bmatrix} \in \mathbb{R}^3
\end{equation}

These use the same computation as $\delta(v,u,w)$ but operate on continuous-valued learned invariants $\phi$, enabling gradient-based optimization. We then learn structure-aware attention weights:
\begin{equation}
\alpha_{vuw} = \sigma(\text{MLP}_{\text{struct}}(d_{vuw}))
\end{equation}
where $\sigma$ bounds output to $[0, 1]$. Structure-aware message passing combines neighbor embeddings through permutation-invariant aggregation:
\begin{equation}
m^{\text{struct}}_{vuw} = \alpha_{vuw} \cdot \text{MLP}_{\text{tri}}(\text{AGG}(h^{(k-1),\text{ISP}}_u, h^{(k-1),\text{ISP}}_w))
\end{equation}

\subsubsection{Architecture Components}

\paragraph{Initialization.} WL features initialized with input attributes; ISP features start at zero:
\begin{equation}
h^{(0),\text{WL}}_v = x_v, \quad h^{(0),\text{ISP}}_v = 0 \quad \forall v \in V
\end{equation}

\paragraph{WL Stream.} At each layer, WL features aggregate from neighbors using combined representations:
\begin{equation}
h^{(k),\text{WL}}_v = \text{UPD}(h^{(k-1),\text{comb}}_v, \text{AGG}_{u \in N(v)} h^{(k-1),\text{comb}}_u)
\end{equation}
where $h^{\text{comb}}_v = [h^{\text{WL}}_v \| h^{\text{ISP}}_v]$ combines both streams.

\paragraph{Gating.} This ensures each node's ISP features are assigned exactly once at its invariant layer. For predefined invariants $\phi : V \to \{1, \ldots, L\}$:
\begin{equation}
\gamma^{(k)}_v = \mathbb{I}[\phi(v) = k] \cdot \mathbb{I}[\|h^{(k-1),\text{ISP}}_v\|_2 = 0]
\end{equation}
where the second indicator prevents overwriting of previously assigned features.

\paragraph{ISP Stream.} For gated nodes, ISP features aggregate structure-weighted messages:
\begin{equation}
h^{(k),\text{TRI}}_v = \begin{cases}
\text{AGG}_{(u,w) \in T(v)} m^{\text{struct}}_{vuw} & \text{if } |T(v)| > 0 \\
0 & \text{otherwise}
\end{cases}
\end{equation}

\begin{equation}
h^{(k),\text{ISP}}_v = \begin{cases}
h^{(k),\text{TRI}}_v & \text{if } \gamma^{(k)}_v = 1, |T(v)| > 0 \\
\text{Embed}(k) & \text{if } \gamma^{(k)}_v = 1, |T(v)| = 0 \\
h^{(k-1),\text{ISP}}_v & \text{otherwise}
\end{cases}
\end{equation}

\paragraph{Output.} Final representation combines both streams:
\begin{equation}
h^{(K)}_v = \text{MLP}_{\text{out}}([h^{(K),\text{WL}}_v \| h^{(K),\text{ISP}}_v])
\end{equation}

\subsubsection{Learnable Invariant Stratification}
While predefined invariants provide fixed hierarchies, we introduce a learnable mechanism for task-adaptive stratification using base invariants $\{\phi_1, \ldots, \phi_m\}$.
\begin{definition}[Learnable Invariant]
Let $\phi_1, \ldots, \phi_m : V \to \mathbb{R}$ be base structural invariants. The learnable invariant is:
\begin{equation}
\phi_{\text{learn}}(v) = (S - 1) \cdot \sigma(\text{MLP}_{\phi}([\phi_1(v), \ldots, \phi_m(v)])) + 1
\end{equation}
where $\text{MLP}_{\phi} : \mathbb{R}^m \to \mathbb{R}$ is learned, $\sigma : \mathbb{R} \to (0, 1)$ is a bounded activation (e.g., sigmoid), and $S \in \mathbb{N}$ controls maximum hierarchical levels.
\end{definition}
\paragraph{Handling Non-Differentiability.} Direct discretization $\phi_{\text{discrete}}(v) = \lfloor \phi_{\text{learn}}(v) \rfloor$ breaks gradient flow. Therefore, we use soft-to-hard training with temperature annealing. Since $\sigma(\cdot) \in (0,1)$, we have $\phi_{\text{learn}}(v) \in (1, S)$, which naturally covers all $S$ layers. During training, we compute soft membership weights for each layer $k \in \{1, \ldots, S\}$:
\begin{equation}
w_k(v) = \frac{\exp(-\beta \cdot |\phi_{\text{learn}}(v) - k|)}{\sum_{k'=1}^{S} \exp(-\beta \cdot |\phi_{\text{learn}}(v) - k'|)}
\end{equation}
where $\beta > 0$ is a temperature parameter that increases during training. ISP features are updated via differentiable weighted aggregation:
\begin{equation}
h_v^{(k), \text{ISP}} = \sum_{k'=1}^{k} w_{k'}(v) \cdot h_v^{(k'), \text{TRI}}
\end{equation}
At the inference phase, we use fully discrete stratification $\phi_{\text{discrete}}(v) = \lfloor \phi_{\text{learn}}(v) \rceil$ (rounding down) with hard gating from Equation~12, assigning each node to exactly one level in $\{1, \ldots, S\}$.
A critical property is that the learned invariant inherits the graph-invariant property from its base components, ensuring that learned stratification respects graph structure and treats structurally equivalent nodes identically.
\begin{restatable}[Structural Consistency]{proposition}{structcons}
\label{prop:learned_inv_invariance}
If each base invariant $\phi_i$ is a graph invariant, then $\phi_{\text{learn}}$ is also a graph invariant.
\end{restatable}
The correspondence between ISP-WL and ISP-GNN is provided in Appendix section \ref{sec:correspondence}. The high-level architecture of ISP-GNN is depicted in Figure \ref{fig:isp_gnn_architecture}.

%% file: sections/theoretical_analysis.tex
\section{Theoretical Analysis}

We establish expressivity, convergence, complexity, and oversmoothing resistance properties. Proofs are deferred to the Appendix.

\vspace{-3   mm}
\subsection{Expressive Power}

\begin{figure}[t]
\centering
\includegraphics[width=0.8\linewidth]{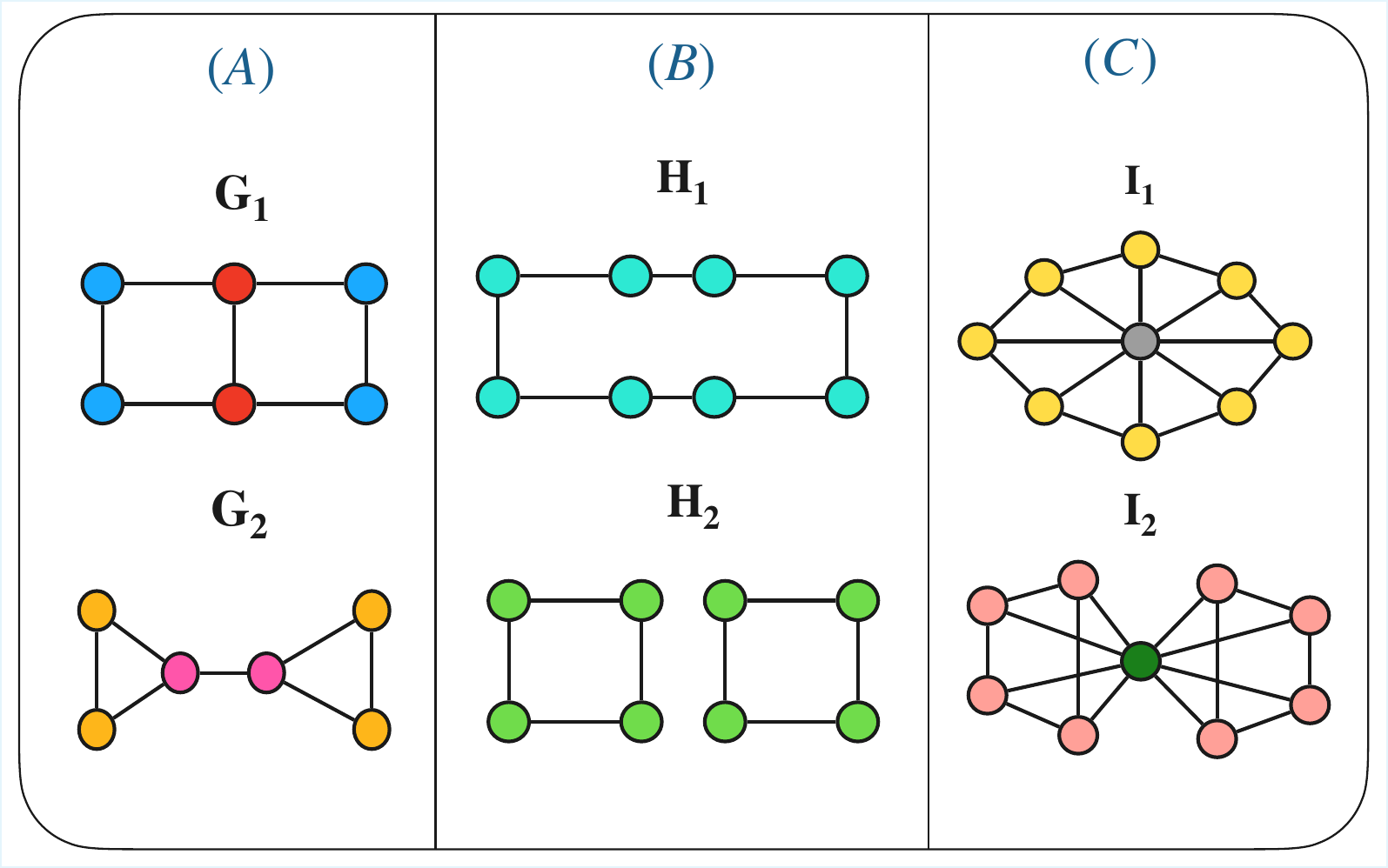}
\caption{Mechanisms by which ISP-WL distinguishes graphs that 1-WL cannot: \textbf{(A) Triangle Aggregation:} $G_1$ and $G_2$ have identical degree sequences but different triangle structures. \textbf{(B) Invariant Stratification:} Triangle-free graphs $H_1$ and $H_2$ differ in structural properties captured by invariants (e.g., local girth). \textbf{(C) Hierarchical Gap Encoding:} $I_1$ and $I_2$ have identical triangle patterns but different invariant hierarchies (e.g., betweenness centrality).}
\label{fig:separation-examples}
\end{figure}

ISP-WL enhances 1-WL through diverse complementary mechanisms: invariant-based stratification, and higher-order aggregation (see Figure \ref{fig:separation-examples}). These mechanisms enable ISP-WL to distinguish graphs beyond the degree sequence barrier.

\begin{restatable}[Enhanced Expressivity Beyond 1-WL]{theorem}{beyondwl}
\label{thm:beyond-1wl}
ISP-WL is strictly more expressive than 1-WL. For any graphs $G_1, G_2$: if 1-WL distinguishes $G_1$ and $G_2$, then ISP-WL distinguishes them. Moreover, there exist non-isomorphic graphs $G_1$ and $G_2$ that 1-WL cannot distinguish, but ISP-WL can.
\end{restatable}

Moreover, ISP-WL provides greater node distinguishability than the standard 1-WL algorithm.

\begin{restatable}[Node Distinguishability]{theorem}{nodedist}
\label{thm:nodedist}
For any nodes $u, v$: if $c_{\text{1-WL}}(u) \neq c_{\text{1-WL}}(v)$, then $c_{\text{ISP-WL}}(u) \neq c_{\text{ISP-WL}}(v)$; moreover, ISP-WL can distinguish nodes even when $c_{\text{1-WL}}(u) = c_{\text{1-WL}}(v)$.
\end{restatable}

We next examine how it compares to higher-order WL variants. Unlike 1-WL and 3-WL \cite{huang2021short}, which have fixed expressivity determined by their refinement rules, ISP-WL's distinguishing power depends on the invariant $\phi$. 

\begin{restatable}[Invariant-Dependent Expressivity]{theorem}{invexpress}
\label{thm:invariant-expressivity}
A coloring $A$ \emph{distinguishes} graphs $G_1$ and $G_2$ if it produces different multisets of colors. Let $A \preceq B$ denote that any pair distinguished by $A$ is also distinguished by $B$, and $A \prec B$ denote strict inequality. The expressivity of ISP-WL is determined by $\phi$:
\begin{enumerate}
    \item If $\phi \preceq \text{1-WL}$, then: $\text{1-WL} \prec \text{ISP-WL}(\phi) \prec \text{3-WL}$
    \item For any $k > 1$, if $\phi$ and $k\text{-WL}$ are incomparable, then $\text{ISP-WL}(\phi)$ and $k\text{-WL}$ are incomparable
    \item If $k\text{-WL} \preceq \phi$ for any $k \geq 1$, then: $k\text{-WL} \preceq \text{ISP-WL}(\phi)$
\end{enumerate}
\end{restatable}

\begin{table*}[t]
\centering
\caption{Graph classification performance comparison across different datasets and methods. Results are reported as accuracy±standard deviation\%. \textcolor{blue}{Best}, \textcolor{red}{second best}, and \textcolor{orange}{third best} results are highlighted. OOM indicates the out-of-memory error.}
\label{tab:results}
\scalebox{1.1}{%  % Adjust scale factor (0.9 = 90% of original size)
\footnotesize
\begin{tabular}{l|ccccccccc}
\toprule
\textbf{Method} & \textbf{MUTAG} & \textbf{PTC\_MR} & \textbf{BZR} & \textbf{DHFR} & \textbf{COX2} & \textbf{PROTEINS} & \textbf{D\&D} & \textbf{IMDB-BINARY} & \textbf{IMDB-MULTI} \\
\midrule
%GCN (2017) & 92.2±4.4 & 68.8±6.2 & 92.6±4.8 & 83.7±3.7 & 88.5±3.8 & 78.8±3.9 & 81.7±4.2 & 79.8±2.3 & 54.7±2.6 \\
%GraphSAGE (2017) & 90.4±6.2 & 69.0±7.2 & 87.0±3.9 & 69.0±5.4 & 85.3±3.5 & 78.4±2.0 & 79.8±4.1 & 78.8±3.3 & 54.9±3.1 \\
GIN (2018)  & 92.8±5.9 & 65.6±6.5 & 91.1±3.4 & 84.9±4.0 & 88.9±2.3 & 78.8±4.1 & 82.2±3.7 & 78.1±3.5 & 54.6±3.0 \\
%GAT (2018) & 90.9±5.8 & 69.4±6.8 & 86.0±6.2 & 75.0±5.1 & 82.3±4.8 & 76.8±2.2 & 78.5±4.6 & 78.3±2.6 & 54.6±2.8 \\
%\midrule
ID-GNN (2021) & \textcolor{red}{97.8±3.7} & 74.4±4.2 & 93.3±4.8 & 86.1±3.8 & 87.8±3.1 & 78.1±3.9 & 81.9±3.8 & 79.3±2.9 & 55.4±2.8 \\
DropGIN (2021) & 92.5±4.9 & 67.1±8.7 & 86.6±5.0 & 67.6±4.8 & 82.9±3.7 & 77.4±2.9 & 79.2±4.3 & 78.2±3.4 & 54.3±2.9 \\
GSN (2022) & 93.1±4.1 & 70.6±5.3 & 93.1±4.3 & 85.7±4.3 & 88.1±3.7 & 78.2±2.8 & 83.9±2.3 & 79.6±3.2 & 54.6±2.5 \\
GraphSNN (2022) & 94.7±1.9 & 70.6±3.1 & 91.1±3.0 & 85.7±3.8 & 86.3±3.3 & 78.4±2.7 & 81.8±3.7 & 78.5±2.3 & 55.3±3.0 \\
GIN-AK+ (2022) & 95.0±6.1 & 74.1±5.9 & 92.8±4.6 & 85.3±4.2 & 88.2±3.6 & 78.9±5.4 & OOM & 79.4±3.3 & 55.2±2.8 \\
KP-GIN (2022) & 95.6±4.4 & \textcolor{red}{76.2±4.5} & 93.5±4.1 & 86.8±3.9 & 88.8±3.2 & 79.5±4.4 & \textcolor{blue}{84.0±3.4} & 79.7±3.0 & 55.5±2.6 \\
NC-GNN (2024) & 92.8±5.0 & 71.8±6.2 & 92.6±4.3 & 80.7±3.4 & 88.4±3.3 & 78.4±3.1 & OOM & 78.4±4.0 & 55.0±2.7 \\
UnionGIN (2024) & 93.6±4.7 & 74.8±6.3 & 92.4±4.2 & 84.2±4.5 & 87.6±3.4 & 78.9±2.5 & 82.3±3.5 & 79.1±3.1 & 55.0±2.7 \\
AC-GIN (2025) & 96.8±3.5 & 75.6±7.1 & 92.9±4.4 & 85.5±3.3 & 88.6±2.7 & 79.5±2.4 & 82.7±3.4 & 79.0±3.4 & 55.1±2.6 \\
BEC-GIN (2025) & 96.1±3.6 & 72.9±5.7 & 92.4±3.6 & \textcolor{orange}{87.5±3.0} & \textcolor{orange}{89.3±3.1} & 79.1±3.7 & 81.8±3.2 & \textcolor{red}{80.8±3.3} & 54.9±3.2 \\
\midrule
ISP-GNN$_\text{Degree}$ & \textcolor{blue}{97.9±2.6} & 75.8±5.9 & 94.1±4.9 & \textcolor{red}{89.5±3.8} & \textcolor{red}{89.7±4.1} & \textcolor{red}{79.8±2.4} & \textcolor{orange}{83.5±2.8} & \textcolor{orange}{80.0±4.1} & \textcolor{red}{55.9±3.1} \\
ISP-GNN$_\text{Core}$ & 96.7±5.3 & 75.5±6.0 & \textcolor{orange}{94.3±5.3} & 86.5±4.8 & 88.9±4.1 & \textcolor{orange}{79.6±2.6} & 83.4±2.4 & 79.8±3.8 & \textcolor{orange}{55.7±2.9} \\
ISP-GNN$_\text{Onion}$ & 96.8±4.8 & \textcolor{orange}{75.9±6.1} & \textcolor{blue}{94.6±3.9} & 85.9±4.5 & 88.7±2.9 & 79.4±2.8 & \textcolor{red}{83.8±3.1} & 79.6±3.6 & 55.6±3.0 \\
ISP-GNN$_\dagger$ & \textcolor{orange}{97.5±2.8} & \textcolor{blue}{76.4±5.7} & \textcolor{red}{94.5±4.5} & \textcolor{blue}{90.1±3.5} & \textcolor{blue}{90.3±3.8} & \textcolor{blue}{80.2±2.3} & 83.5±3.4 & \textcolor{blue}{81.2±3.8} & \textcolor{blue}{56.3±2.8} \\
\bottomrule
\end{tabular}
}% End of scalebox
\end{table*}

\subsection{Convergence and Complexity}

The single assignment property established in Section~3.1.2 ensures both correctness and efficiency by guaranteeing that each node's structural information is computed exactly once.

\begin{restatable}[Convergence]{theorem}{convergence}
\label{thm:convergence}
ISP-WL terminates in at most $K = \max(K_{\text{WL}}, L) \leq |V|$ iterations, where $L$ is the number of distinct invariant values and $K_{\text{WL}}$ is the standard 1-WL convergence iterations.
\end{restatable}

The following theorem establishes computational efficiency on sparse graphs, showing that ISP-WL maintains the same asymptotic complexity as standard 1-WL.

\begin{restatable}[Time Complexity]{theorem}{complexity}
\label{thm:complexity}
ISP-WL has time complexity $O(K(|V| + |E|) + T)$ where $T = \sum_v |T(v)|$ is the total number of triangles. For sparse graphs with bounded degeneracy $d$, we have $T = O(d|E|)$, giving overall complexity $O(K(|V| + |E|))$ when $d$ is constant, matching standard 1-WL.
\end{restatable}

This result applies broadly to real-world networks, which typically have a bounded degeneracy \cite{allen2023universality}, such as social networks, molecular graphs, and citation networks. 

\subsection{Oversmoothing Resistance}

Deep GNNs often face oversmoothing, where node representations become indistinguishable as depth increases \cite{zhang2023comprehensive}. ISP-GNN addresses this by maintaining structural information through a fixed ISP stream, serving as a structural anchor.

\begin{restatable}[Resistance to Feature Collapse]{theorem}{oversmoothing}
\label{thm:oversmoothing}
Let $\bar{h}_v^{\text{ISP}} := h_v^{(\phi(v)),\text{ISP}}$ denote the fixed structural embedding at layer $\phi(v)$. For nodes with distinct structural embeddings $\bar{h}_u^{\text{ISP}} \neq \bar{h}_v^{\text{ISP}}$, the final combined representation satisfies:
\[
[h_u^{(K),\text{WL}} \| \bar{h}_u^{\text{ISP}}] \neq [h_v^{(K),\text{WL}} \| \bar{h}_v^{\text{ISP}}] \quad \forall K
\]
even as WL features oversmooth. If $\text{MLP}_{\text{out}}$ is injective, final node outputs remain distinct across all depths.
\end{restatable}

Thus, the depth-invariant structural information $\bar{h}_v^{\text{ISP}}$ prevents feature collapse.

%% file: sections/experiments.tex
\section{Experiments}

We conduct a comprehensive evaluation of ISP-GNN across multiple graph learning domains. First, we assess expressive power through standard benchmarks: graph classification for distinguishing graph-level structures and node classification for capturing local neighborhood patterns. Second, we evaluate the model's ability to capture temporal dynamics through influence estimation tasks \cite{xia2021deepis,ling2023deep}, where information cascades through networks following hierarchical propagation patterns that our invariant-stratified framework is designed to model. 

For our main experiments, we primarily employ degree, k-core \cite{malliaros2020core}, and onion decomposition \cite{hebert2016multi} as invariants, offering computational efficiency ranging from $O(|V|)$ to $O(|E|)$ while providing diverse structural perspectives. Additional invariants are used in ablation studies to demonstrate the effectiveness of our approach. Additional details on these invariants are provided in the Appendix \ref{sec:invariants}.

\subsection{Experimental Setup}

\paragraph{Datasets} We evaluate graph classification using the TU \cite{morris2020tudataset} and OGB \cite{hu2020open} dataset benchmarks, which include datasets from the molecular, bioinformatics, and social network domains. For node classification, we evaluate datasets with homophilic and heterophilic label patterns from CitationFull \cite{bojchevski2018deep},  Amazon \cite{shchur2018pitfalls}, WebKB \cite{pei2020geom}, and Heterophilic \cite{platonovcritical} benchmarks.  For influence estimation, we employ four real-world
social networks \cite{bojchevski2018deep,rossi2015network}.

\paragraph{Baselines} For graph classification, we compare ISP-GNN with traditional GNNs, which are upper bounded by 1-WL (GIN \cite{xu2018powerful}), and GNNs that are designed to incorporate higher-order structures to exceed 1-WL (ID-GNN \cite{wijesinghe2022new}, DropGNN \cite{papp2021dropgnn}, GSN \cite{bouritsas2022improving}, GraphSNN \cite{wijesinghe2022new}, GIN-AK+ \cite{zhaostars2022}, KP-GIN \cite{feng2022powerful}, UnionSNN \cite{xu2024union}, NC-GNN \cite{liuempowering}, AC-GNN \cite{hevapathige2025curvature}, BEC-GIN \cite{hevapathige2025depth}). These baselines span four categories: (i) Structure injection (GSN, GraphSNN, UnionSNN), (ii) subgraph enumeration (GIN-AK+, KP-GIN), (iii) structural augmentation (ID-GNN, DropGNN, NC-GNN), and (iv) geometric methods using curvature (AC-GNN, BEC-GIN). For node classification, we utilize traditional GNN architectures: GCN \cite{kipf2017semi}, GAT \cite{velivckovic2018graph}, GraphSAGE \cite{hamilton2017inductive}, along with modern baselines: Dir-GNN \cite{rossi2024edge}, a direction-aware model, and BEC-GCN \cite{hevapathige2025depth}, a curvature-based model. For influence estimation, we evaluate against traditional GNNs (GIN \cite{xu2018powerful}, GCN \cite{kipf2017semi}, GAT \cite{velivckovic2018graph}, GraphSAGE \cite{hamilton2017inductive}), the direction-aware DirGNN architecture \cite{rossi2024edge}, task-specific influence estimation methods (SGNN \cite{kumar2022influence}, DeepIM \cite{ling2023deep}, DeepIS \cite{xia2021deepis}, GLIE \cite{panagopoulos2024learning}), and diffusion-oriented models (APPNP \cite{gasteiger2018predict}, ODNet \cite{zhouodnet}, UniGO \cite{li2025unigo}) designed using the indutive bias of opinion propagation analogous to influence estimation.

\paragraph{Evaluation Settings} For graph classification, we use two setups. For TU datasets, we follow \citet{feng2022powerful} with 10-fold cross-validation, reporting mean and standard deviation of the best accuracy. For OGB datasets, we adopt \citet{hu2020open}'s setup, focusing on the mean and standard deviation of ROC-AUC scores over 10 random seeds. In node classification, we apply a 60/20/20 split per \citet{zheng2023finding} and report mean accuracy across 10 initializations. For influence estimation, we assess three diffusion models: Linear Threshold (LT), Independent Cascade (IC), and Susceptible–Infected–Susceptible (SIS), with the initial activation node set at 10\% of the nodes. We follow \citet{ling2023deep}'s setup, reporting average MAE and standard deviation over 10 folds. Baseline results are reported from prior work when available; otherwise, we replicate them using specified hyperparameters.

Additional details regarding downstream tasks, baselines,
dataset statistics, model hyperparameters, and implementation details are provided in the
Appendix section \ref{sec:exp_des}.

\begin{table*}[t]
\centering
\caption{Influence estimation performance comparison across different datasets and methods. Results are reported as MAE$\pm$standard deviation. Lower values indicate better performance. \textcolor{blue}{Best}, \textcolor{red}{second best}, and \textcolor{orange}{third best} results are highlighted.}
\label{tab:influence_estimation}
\footnotesize
\scalebox{0.95}{
\begin{tabular}{l|ccc|ccc|ccc|ccc}
\toprule
\multirow{2}{*}{Method} & \multicolumn{3}{c|}{Jazz} & \multicolumn{3}{c|}{Cora-ML} & \multicolumn{3}{c|}{Network Science} & \multicolumn{3}{c}{Power Grid} \\
\cmidrule(lr){2-4} \cmidrule(lr){5-7} \cmidrule(lr){8-10} \cmidrule(lr){11-13}
 & IC & LT & SIS & IC & LT & SIS & IC & LT & SIS & IC & LT & SIS \\
\midrule
GCN (2017) & 0.233$_{\pm.010}$ & 0.199$_{\pm.006}$ & 0.344$_{\pm.023}$ & 0.277$_{\pm.007}$ & 0.255$_{\pm.008}$ & 0.365$_{\pm.065}$ & 0.270$_{\pm.019}$ & 0.190$_{\pm.012}$ & 0.180$_{\pm.007}$ & 0.313$_{\pm.024}$ & 0.335$_{\pm.023}$ & 0.207$_{\pm.001}$ \\
GraphSAGE (2017) & 0.201$_{\pm.028}$ & 0.120$_{\pm.004}$ & 0.301$_{\pm.018}$ & 0.255$_{\pm.010}$ & 0.203$_{\pm.019}$ & 0.222$_{\pm.051}$ & 0.241$_{\pm.010}$ & 0.112$_{\pm.005}$ & 0.102$_{\pm.005}$ & 0.313$_{\pm.024}$ & 0.341$_{\pm.018}$ & 0.133$_{\pm.001}$ \\
GIN (2018) & 0.148$_{\pm.012}$ & 0.174$_{\pm.029}$ & 0.344$_{\pm.013}$ & 0.247$_{\pm.001}$ & 0.305$_{\pm.008}$ & 0.225$_{\pm.019}$ & 0.277$_{\pm.005}$ & 0.265$_{\pm.006}$ & 0.140$_{\pm.007}$ & 0.276$_{\pm.001}$ & 0.471$_{\pm.002}$ & 0.184$_{\pm.029}$ \\
GAT (2018) & 0.342$_{\pm.005}$ & 0.156$_{\pm.100}$ & \textcolor{red}{0.288$_{\pm.017}$} & 0.352$_{\pm.004}$ & 0.192$_{\pm.010}$ & 0.208$_{\pm.008}$ & 0.274$_{\pm.002}$ & 0.114$_{\pm.008}$ & 0.123$_{\pm.013}$ & 0.331$_{\pm.002}$ & 0.280$_{\pm.015}$ & 0.133$_{\pm.001}$ \\
APPNP (2018) & 0.200$_{\pm.006}$ & 0.124$_{\pm.003}$ & 0.357$_{\pm.059}$ & 0.265$_{\pm.022}$ & 0.220$_{\pm.031}$ & 0.321$_{\pm.042}$ & 0.248$_{\pm.006}$ & 0.084$_{\pm.006}$ & 0.100$_{\pm.012}$ & 0.290$_{\pm.006}$ & \textcolor{orange}{0.189$_{\pm.011}$} & 0.132$_{\pm.001}$ \\
DeepIS (2021) & 0.151$_{\pm.003}$ & 0.219$_{\pm.002}$ & 0.434$_{\pm.003}$ & 0.203$_{\pm.001}$ & 0.301$_{\pm.005}$ & 0.304$_{\pm.001}$ & 0.223$_{\pm.001}$ & 0.306$_{\pm.001}$ & 0.256$_{\pm.001}$ & \textcolor{orange}{0.206$_{\pm.001}$} & 0.374$_{\pm.001}$ & 0.251$_{\pm.001}$ \\
SGNN (2022) & 0.183$_{\pm.004}$ & 0.164$_{\pm.014}$ & 0.330$_{\pm.007}$ & 0.210$_{\pm.003}$ & {0.134}$_{\pm.004}$ & 0.211$_{\pm.006}$ & 0.213$_{\pm.003}$ & \textcolor{orange}{0.049$_{\pm.004}$} & 0.127$_{\pm.000}$ & 0.257$_{\pm.002}$ & 0.192$_{\pm.001}$ & 0.175$_{\pm.004}$ \\
DeepIM (2023) & 0.178$_{\pm.002}$ & 0.134$_{\pm.014}$ & 0.383$_{\pm.010}$ & 0.210$_{\pm.002}$ & 0.271$_{\pm.010}$ & 0.270$_{\pm.006}$ & 0.216$_{\pm.003}$ & 0.118$_{\pm.003}$ & 0.135$_{\pm.004}$ & 0.258$_{\pm.003}$ & 0.331$_{\pm.002}$ & 0.205$_{\pm.005}$ \\
GLIE (2024) & \textcolor{orange}{0.136$_{\pm.003}$} & \textcolor{red}{0.055$_{\pm.028}$} & 0.454$_{\pm.062}$ & 0.199$_{\pm.041}$ & 0.286$_{\pm.016}$ & 0.205$_{\pm.029}$ & \textcolor{red}{0.163$_{\pm.031}$} & 0.160$_{\pm.063}$ & 0.103$_{\pm.023}$ & \textcolor{red}{0.183$_{\pm.009}$} & 0.384$_{\pm.020}$ & 0.132$_{\pm.026}$ \\
DirGNN (2024) & 0.205$_{\pm.016}$ & 0.124$_{\pm.002}$ & 0.379$_{\pm.032}$ & 0.255$_{\pm.006}$ & 0.193$_{\pm.012}$ & 0.229$_{\pm.025}$ & 0.236$_{\pm.011}$ & 0.084$_{\pm.017}$ & 0.142$_{\pm.019}$ & 0.287$_{\pm.010}$ & 0.257$_{\pm.062}$ & 0.132$_{\pm.001}$ \\
ODNet (2024) & 0.180$_{\pm.003}$ & \textcolor{blue}{0.053$_{\pm.003}$} & 0.322$_{\pm.006}$ & 0.232$_{\pm.001}$ & 0.210$_{\pm.004}$ & 0.196$_{\pm.002}$ & 0.255$_{\pm.002}$ & 0.104$_{\pm.015}$ & 0.106$_{\pm.005}$ & 0.274$_{\pm.001}$ & 0.296$_{\pm.002}$ & 0.145$_{\pm.001}$ \\
UniGO (2025) & 0.192$_{\pm.013}$ & 0.159$_{\pm.020}$ & 0.335$_{\pm.002}$ & 0.255$_{\pm.001}$ & 0.155$_{\pm.019}$ & 0.212$_{\pm.001}$ & 0.201$_{\pm.001}$ & 0.110$_{\pm.049}$ & 0.108$_{\pm.023}$ & 0.231$_{\pm.002}$ & 0.235$_{\pm.005}$ & 0.206$_{\pm.023}$ \\
\midrule
ISP-GNN$_\text{Degree}$ & \textcolor{blue}{0.132$_{\pm.001}$} & 0.137$_{\pm.028}$ & 0.315$_{\pm.013}$ & \textcolor{blue}{0.162$_{\pm.001}$} & \textcolor{blue}{0.106$_{\pm.007}$} & 0.191$_{\pm.012}$ & 0.205$_{\pm.001}$ & \textcolor{blue}{0.045$_{\pm.008}$} & 0.100$_{\pm.005}$ & 0.245$_{\pm.004}$ & 0.195$_{\pm.004}$ & \textcolor{red}{0.126$_{\pm.006}$} \\
ISP-GNN$_\text{Core}$ & 0.147$_{\pm.011}$ & 0.162$_{\pm.009}$ & 0.314$_{\pm.012}$ & 0.179$_{\pm.002}$ & 0.152$_{\pm.007}$ & \textcolor{blue}{0.184$_{\pm.004}$} & 0.210$_{\pm.004}$ & 0.051$_{\pm.009}$ & \textcolor{blue}{0.094$_{\pm.001}$} & 0.240$_{\pm.002}$ & \textcolor{blue}{0.181$_{\pm.004}$} & \textcolor{orange}{0.127$_{\pm.002}$} \\
ISP-GNN$_\text{Onion}$ & 0.156$_{\pm.009}$ & 0.155$_{\pm.003}$ & \textcolor{orange}{0.294$_{\pm.014}$} & \textcolor{orange}{0.176$_{\pm.002}$} & \textcolor{orange}{0.134$_{\pm.006}$} & \textcolor{red}{0.187$_{\pm.004}$} & \textcolor{blue}{0.160$_{\pm.003}$} & 0.060$_{\pm.002}$ & \textcolor{red}{0.095$_{\pm.001}$} & 0.263$_{\pm.001}$ & 0.222$_{\pm.005}$ & \textcolor{blue}{0.120$_{\pm.002}$} \\
ISP-GNN$_\dagger$ & \textcolor{red}{0.135$_{\pm.003}$} & \textcolor{orange}{0.058$_{\pm.005}$} & \textcolor{blue}{0.285$_{\pm.011}$} & \textcolor{red}{0.168$_{\pm.003}$} & \textcolor{red}{0.112$_{\pm.009}$} & \textcolor{orange}{0.189$_{\pm.007}$} & \textcolor{orange}{0.168$_{\pm.005}$} & \textcolor{red}{0.048$_{\pm.006}$} & \textcolor{orange}{0.096$_{\pm.003}$} & \textcolor{blue}{0.180$_{\pm.003}$} & \textcolor{red}{0.186$_{\pm.006}$} & 0.124$_{\pm.004}$ \\
\bottomrule
\end{tabular}
}
\end{table*}

\subsection{Main Results}

In our experiments, we use the notation ISP-GNN$_\phi$ to refer to the model instantiated with a specific predefined invariant $\phi$, and ISP-GNN$_\dagger$ to denote the variant with learnable invariant stratification.

\subsubsection{Graph Classification} 

Table~\ref{tab:results} presents graph classification performance across nine TU benchmark datasets. ISP-GNN variants consistently outperform or provide competitive results against both traditional GNN architectures and expressive GNN baselines. Notably, the learnable ISP-GNN$_\dagger$ variant shows a clear edge over pre-defined variants on the majority of datasets, achieving top performance on five benchmarks. We believe this demonstrates the advantage of learning task-specific structural representations. Different predefined variants perform differently across graph types, suggesting that structural-invariant choice correlates with specific graph characteristics.

\subsubsection{Influence Estimation}

Table~\ref{tab:influence_estimation} evaluates the influence estimation performance of ISP-GNN on four social network datasets. As shown in the results, ISP-GNN variants consistently achieve top-3 performance, demonstrating the effectiveness of hierarchical structural encoding for modelling cascade dynamics. The learnable ISP-GNN$_\dagger$ variant demonstrates strong performance across diffusion models, suggesting that learned structural stratification effectively captures diffusion-relevant patterns that vary across network topologies and propagation mechanisms. Different predefined variants excel in complementary scenarios, with degree-based stratification capturing direct influence spread patterns while core and onion-based stratification better model community structure effects on cascades. Overall, ISP-GNN competes favourably with task-specific influence estimation methods while maintaining generalizability across diverse diffusion models and graph structures, demonstrating that our method provides benefits beyond static graph property prediction.

\subsubsection{Node Classification}
Table~\ref{tab:node_classification} evaluates the node classification performance of ISP-GNN across eight benchmark datasets. We augment baseline architectures by combining learnable structural color embeddings with baseline embeddings to produce ISP variants. ISP-GNN$_\dagger$ consistently improves performance, demonstrating the effectiveness of hierarchical structural encoding for node-level prediction tasks. Our approach shows substantial improvements on heterophilic datasets, where traditional message-passing struggles, suggesting that ISP color embeddings capture structural patterns that complement feature-based neighborhood aggregation when local homophily assumptions break down. On homophilic networks, ISP-GNN$_\dagger$ provides modest but consistent improvements, indicating that structural stratification adds valuable positional information even when node features and labels are strongly correlated.

\begin{table*}[htbp]
\centering
\caption{Node classification performance comparison across different datasets and methods. Results are reported as accuracy ± standard deviation (\%). $\Delta$ represents the maximum percentage improvement of ISP variants over their respective baseline.}
\footnotesize
\scalebox{1.1}{
\begin{tabular}{l|c|c|c|c|c|c|c|c}
\hline
 \textbf{Method} & \textbf{Cora ML} & \textbf{Citeseer} & \textbf{Pubmed} & \textbf{DBLP} & \textbf{Film} & \textbf{Cornell} & \textbf{Wisconsin} & \textbf{Texas} \\
 \hline
 GraphSAGE (2017) & 86.52 ± 1.32 & 76.04 ± 1.30 & 88.45 ± 0.50 & 86.16 ± 0.50 & 34.23 ± 0.99 & 75.95 ± 5.01 & 81.18 ± 5.56 & 82.43 ± 6.14 \\
ISP-GraphSAGE$_\dagger$ & 87.23 ± 1.78 & 77.91 ± 0.92 & 88.51 ± 0.58 & 83.93 ± 0.60 & 35.89 ± 1.18 & 78.51 ± 5.98 & 87.38 ± 5.05 & 83.61 ± 6.96 \\
$\Delta$ (\%) & \textcolor{blue}{+0.98} & \textcolor{blue}{+2.89} & \textcolor{blue}{+0.07} & -2.10 & \textcolor{blue}{+6.31} & \textcolor{blue}{+6.45} & \textcolor{blue}{+7.64} & \textcolor{blue}{+7.39} \\
\hline
GCN (2018) & 87.07 ± 1.21 & 76.68 ± 1.64 & 86.74 ± 0.47 & 83.93 ± 0.34 & 30.26 ± 0.79 & 55.14 ± 8.46 & 61.60 ± 7.00 & 60.00 ± 6.45 \\
ISP-GCN$_\dagger$ & 87.29 ± 1.20 & 78.05 ± 1.70 & 89.01 ± 0.42 & 84.13 ± 0.53 & 35.10 ± 1.03 & 64.47 ± 8.94 & 66.75 ± 6.35 & 79.67 ± 3.53 \\
$\Delta$ (\%) & \textcolor{blue}{+0.25} & \textcolor{blue}{+1.89} & \textcolor{blue}{+2.62} & \textcolor{blue}{+0.82} & \textcolor{blue}{+15.99} & \textcolor{blue}{+22.71} & \textcolor{blue}{+8.36} & \textcolor{blue}{+32.78} \\
\hline
GAT (2018) & 84.12 ± 0.55 & 75.46 ± 1.72 & 87.24 ± 0.55 & 80.61 ± 1.21 & 26.28 ± 1.73 & 53.64 ± 11.1 & 60.00 ± 11.0 & 61.21 ± 8.17 \\
ISP-GAT$_\dagger$ & 87.14 ± 1.18 & 78.02 ± 1.32 & 88.05 ± 0.91 & 84.25 ± 0.63 & 32.67 ± 3.68 & 71.70 ± 4.47 & 62.75 ± 12.55 & 64.26 ± 2.44 \\
$\Delta$ (\%) & \textcolor{blue}{+3.83} & \textcolor{blue}{+4.04} & \textcolor{blue}{+1.09} & \textcolor{blue}{+4.52} & \textcolor{blue}{+26.33} & \textcolor{blue}{+33.67} & \textcolor{blue}{+6.25} & \textcolor{blue}{+4.98} \\
\hline
DIRGNN (2024) & 85.66 ± 0.31 & 77.71 ± 0.78 & 86.94 ± 0.55 & 81.22 ± 0.54 & 35.76 ± 1.68 & 76.51 ± 6.14 & 80.50 ± 5.50 & 76.25 ± 6.31 \\
ISP-DIRGNN$_\dagger$ & 86.38 ± 1.09 & 77.86 ± 1.73 & 89.16 ± 0.78 & 83.53 ± 0.54 & 36.27 ± 1.74 & 81.28 ± 4.34 & 88.63 ± 4.05 & 84.26 ± 3.13 \\
$\Delta$ (\%) & \textcolor{blue}{+1.68} & \textcolor{blue}{+1.12} & \textcolor{blue}{+2.55} & \textcolor{blue}{+2.91} & \textcolor{blue}{+1.85} & \textcolor{blue}{+6.23} & \textcolor{blue}{+11.49} & \textcolor{blue}{+13.30} \\
\hline
BEC-GCN (2025) & 88.65 ± 1.32 & 78.14 ± 1.31 & 87.00 ± 0.28 & 85.20 ± 0.48 & 34.34 ± 2.08 & 65.96 ± 9.03 & 66.25 ± 4.55 & 68.85 ± 11.02 \\
ISP-BEC-GCN$_\dagger$ & 89.11 ± 1.70 & 78.58 ± 1.68 & 88.21 ± 0.92 & 86.43 ± 0.87  & 35.09 ± 1.88  & 68.09 ± 9.06 & 73.75 ± 4.65 & 70.49 ± 10.89 \\
$\Delta$ (\%) & \textcolor{blue}{+0.51} & \textcolor{blue}{+0.56} & \textcolor{blue}{+1.39} & \textcolor{blue}{+1.44} & \textcolor{blue}{+2.18} & \textcolor{blue}{+3.23} & \textcolor{blue}{+11.32} & \textcolor{blue}{+2.38} \\
\hline
\end{tabular}
}
\label{tab:node_classification}
\end{table*}

\subsection{Ablation Studies}

\subsubsection{Oversmoothing Analysis}

Figure~\ref{fig:oversmoothing_performance} examines the robustness of ISP-GNN to oversmoothing as network depth increases. On both homophilic (Citeseer) and heterophilic (Texas) datasets, baseline GCN and GAT models exhibit severe performance degradation beyond 8 layers, with accuracy dropping dramatically due to oversmoothing. In contrast, ISP-GNN variants maintain substantially higher performance at deeper architectures, demonstrating that structural color embeddings alleviate the oversmoothing problem by preserving node distinguishability even when feature representations converge.

\begin{figure}[H]
    \centering
    
    % Legend at the top
    \begin{center}
        \begin{tikzpicture}
            % GCN
            \draw[red, line width=1.5pt] (0,0) -- (0.5,0) 
                node[pos=0.5] {\pgfuseplotmark{o}};
            \node[right, font=\small] at (0.6,0) {GCN};
            
            % GAT  
            \draw[blue, line width=1.5pt] (1.5,0) -- (2.0,0) 
                node[pos=0.5] {\pgfuseplotmark{square}};
            \node[right, font=\small] at (2.1,0) {GAT};
            
            % AD-GCN
            \draw[red, line width=1.5pt, dashed] (3.0,0) -- (3.5,0) 
                node[pos=0.5] {\pgfuseplotmark{triangle}};
            \node[right, font=\small] at (3.6,0) {ISP-GCN};
            
            % AD-GAT
            \draw[blue, line width=1.5pt, dashed] (5.2,0) -- (5.7,0) 
                node[pos=0.5] {\pgfuseplotmark{diamond}};
            \node[right, font=\small] at (5.8,0) {ISP-GAT};
        \end{tikzpicture}
    \end{center}
    
    \vspace{0.3cm}
    
    \begin{minipage}[b]{0.49\linewidth}
        \centering
        \resizebox{\linewidth}{!}{
        \begin{tikzpicture}
            \begin{axis}[    
                xlabel={Layer Depth},    
                ylabel={Accuracy\%},
                xlabel style={font=\normalsize},
                ylabel style={font=\normalsize},
                xmin=0, xmax=6,    
                ymin=0, ymax=85,    
                xtick={0, 1, 2, 3, 4, 5, 6},
                xticklabels={1, 2, 4, 8, 16, 32, 64},    
                ytick={0, 20, 40, 60, 80},    
                ymajorgrids=true,    
                grid style=dashed,
                ylabel style={font=\bfseries},
                xlabel style={font=\bfseries}
            ]
            \addplot[    
                color=red,    
                mark=o,
                line width=1.5pt,
                ]
                coordinates {(0, 74.24)(1, 76.68)(2, 76.17)(3, 60.78)(4, 11.38)(5, 1.71)(6, 1.71)
                };
            
            \addplot[    
                color=blue,    
                mark=square,
                line width=1.5pt,
                ]
                coordinates {(0, 73.86)(1, 75.46)(2, 74.35)(3, 60.97)(4, 14.86)(5, 1.9)(6, 1.59)
                };
            
            \addplot[    
                color=red,    
                mark=triangle,
                line width=1.5pt,
                dashed,
                ]
                coordinates {(0, 77.29)(1, 78.05)(2, 77.68)(3, 78.17)(4, 74.22)(5, 38.74)(6, 29.74)
                };
            
            \addplot[    
                color=blue,    
                mark=diamond,
                line width=1.5pt,
                dashed,
                ]
                coordinates {(0, 77.14)(1, 78.02)(2, 78.31)(3, 76.13)(4, 74.22)(5, 40.52)(6, 26.19)
                };
            \end{axis}
        \end{tikzpicture}
        }
        \caption*{(a) Citeseer}
    \end{minipage}\hfill
    \begin{minipage}[b]{0.49\linewidth}
        \centering
        \resizebox{\linewidth}{!}{
        \begin{tikzpicture}
            \begin{axis}[    
                xlabel={Layer Depth},    
                ylabel={Accuracy\%},
                xlabel style={font=\normalsize},
                ylabel style={font=\normalsize},
                xmin=0, xmax=6,    
                ymin=0, ymax=100,    
                xtick={0, 1, 2, 3, 4, 5, 6},
                xticklabels={1, 2, 4, 8, 16, 32, 64},    
                ytick={0, 20, 40, 60, 80, 100},    
                ymajorgrids=true,    
                grid style=dashed,
                ylabel style={font=\bfseries},
                xlabel style={font=\bfseries}
            ]
            \addplot[    
                color=red,    
                mark=o,
                line width=1.5pt,
                ]
                coordinates {(0, 59.01)(1, 60)(2, 59.55)(3, 57.48)(4, 39.91)(5, 25.5)(6, 27.39)
                };
            
            \addplot[    
                color=blue,    
                mark=square,
                line width=1.5pt,
                ]
                coordinates {(0, 59.54)(1, 61.21)(2, 58.28)(3, 57.65)(4, 41.71)(5, 18.64)(6, 6.76)
                };
            
            \addplot[    
                color=red,    
                mark=triangle,
                line width=1.5pt,
                dashed,
                ]
                coordinates {(0, 74.26)(1, 79.67)(2, 80.33)(3, 80.33)(4, 73.77)(5, 73.77)(6, 68.85)
                };
            
            \addplot[    
                color=blue,    
                mark=diamond,
                line width=1.5pt,
                dashed,
                ]
                coordinates {(0, 69.34)(1, 64.26)(2, 65.57)(3, 67.21)(4, 65.57)(5, 63.93)(6, 68.85)
                };
            \end{axis}
        \end{tikzpicture}
        }
        \caption*{(b) Texas}
    \end{minipage}
    
    \caption{ Oversmoothing comparison}
    \label{fig:oversmoothing_performance}
\end{figure}

\subsubsection{Component Analysis}

To assess each mechanism's contribution, we ablated core components of ISP-GNN: invariant-based stratification, triangle aggregation, and hierarchical structural heterogeneity encoding, evaluating them on four graph classification benchmarks. Results in Table ~\ref{tab:ablation} show triangle aggregation is the most critical element, with its removal leading to the largest performance drop, particularly in social networks. Invariant stratification follows as the second most important, while hierarchical gaps provide consistent, modest improvements. Overall, the combination of all three mechanisms yields greater performance than any individual component.

\begin{center}
\captionof{table}{Component analysis on graph classification. Results reported as mean accuracy±standard deviation (\%). The best results are highlighted in \textbf{bold}.}
\label{tab:ablation}
\scalebox{0.75}{
\begin{tabular}{lcccc}
\toprule
\textbf{Variant} & \textbf{MUTAG} & \textbf{PTC-MR} & \textbf{PROTEINS} & \textbf{IMDB-BINARY} \\
\midrule
WL Only & 92.8 ± 5.9 & 65.6 ± 6.5 & 78.8 ± 4.1 & 78.1 ± 3.5 \\
w/o Stratification & 95.2 ± 3.4 & 73.5 ± 6.2 & 79.2 ± 2.8 & 79.4 ± 4.2 \\
w/o Triangle Aggregation & 93.5 ± 4.1 & 71.5 ± 6.8 & 79.0 ± 3.1 & 78.8 ± 4.5 \\
w/o Hierarchical Gap & 94.2 ± 3.6 & 72.8 ± 6.4 & 79.3 ± 2.9 & 79.2 ± 4.3 \\
ISP-GNN$_\dagger$ & \textbf{97.5 ± 2.8} & \textbf{76.4 ± 5.7} & \textbf{80.2 ± 2.3} & \textbf{80.5 ± 3.9} \\
\bottomrule
\end{tabular}
}
\end{center}

\subsubsection{Scalability Analysis}

To assess the efficiency of ISP-GNN on large datasets, we evaluate its runtime and performance on ogbg-molpcba, a molecular property prediction benchmark with 437,929 graphs. For fair comparison, we use 5 layers with an embedding dimension of 300 across all models. Results are shown in Figure~\ref{fig:scalability}.

\begin{figure}[htbp]
\centering
\includegraphics[width=\columnwidth]{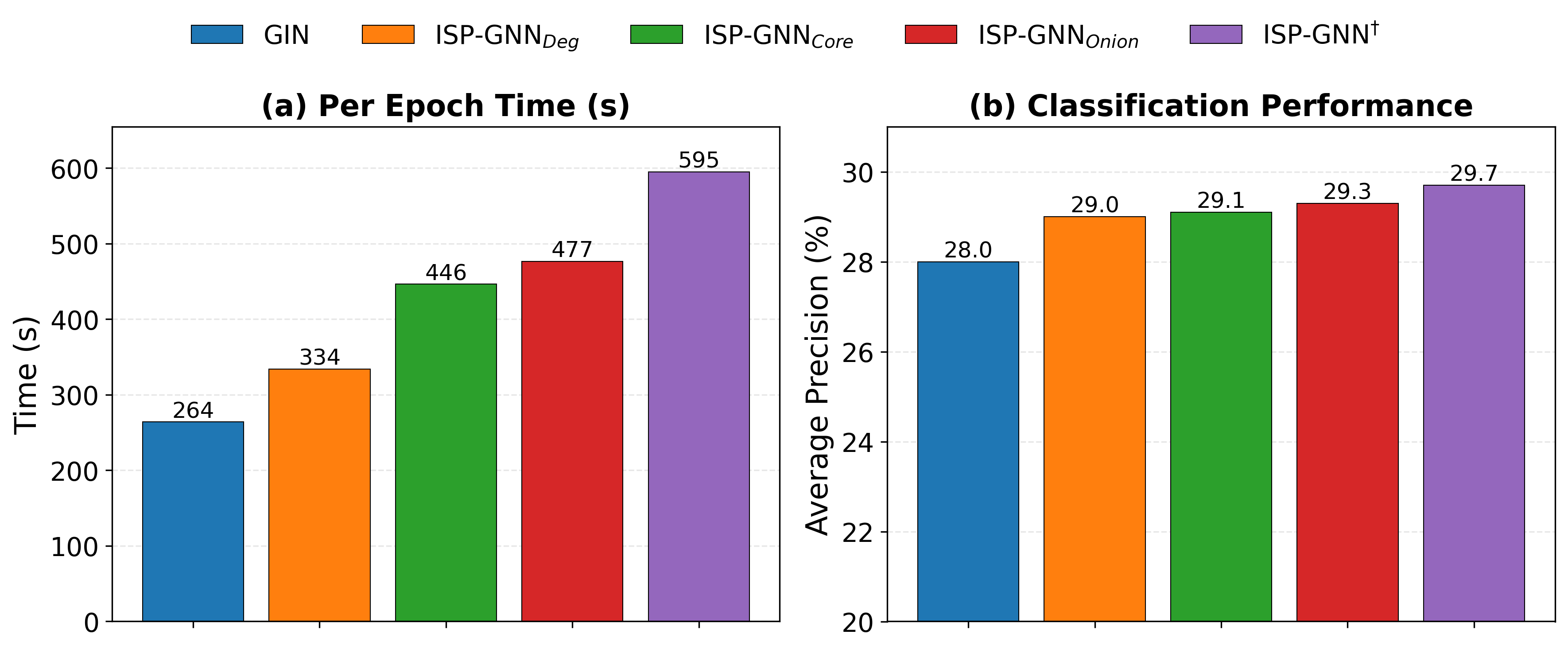}
\caption{Runtime and performance on ogbg-molpcba dataset. (a) Per epoch time in seconds. (b) Average precision score. Results show ISP-GNN with different structural invariants.}
\label{fig:scalability}
\end{figure}

ISP-GNN demonstrates practical scalability on large molecular graphs. Degree-based stratification adds only 26\% overhead to baseline GIN due to efficient invariant computation. Core decomposition incurs higher costs due to iterative peeling needed to identify k-cores, while onion decomposition further increases runtime by creating finer-grained hierarchical layers that require additional processing iterations. The learnable variant trades the highest computational cost for the best performance by optimising over multiple invariants. This spectrum enables practitioners to balance efficiency and expressivity based on application needs, with all variants remaining tractable on large datasets.

\subsubsection{Expressivity Comparison}

We evaluate the expressive power of ISP-WL using the BREC benchmark \cite{wang2024empirical}, reporting average processing time per graph pair alongside accuracy. Results (Table~\ref{tab:expressivity}) demonstrate that ISP-WL consistently outperforms standard WL variants such as 1-WL and SPD-WL \cite{wang2024empirical}. Performance scales with invariant sophistication: degree shows minimal gains, betweenness centrality (BC) provides moderate improvement, while complex invariants like average neighborhood clustering (ANC) achieve the strongest performance. Notably, ISP-WL variants maintain practical efficiency. Even ISP-WL$_{\text{ANC}}$ processes graphs seven times faster than 3-WL while achieving two-thirds of its expressivity. This invariant dependent expressivity directly validates Theorem~\ref{thm:invariant-expressivity}. Critically, when the invariant is incomparable to 3-WL, ISP-WL can distinguish graphs that 3-WL cannot while maintaining lower complexity, as demonstrated on regular graphs. While ISP-WL remains upper bounded by 3-WL with generic invariants on standard BREC, these results confirm that strategic invariant selection enables ISP-WL to complement higher-order methods efficiently.

\begin{table}[htbp]
\centering
\caption{Expressivity evaluation on BREC benchmark}
\label{tab:expressivity}
\small
\setlength{\tabcolsep}{3pt}
\begin{tabular}{@{}cc@{}}
\begin{minipage}{0.48\linewidth}
\centering
\subcaption{Overall Performance}
\label{tab:brec}
\vspace{2mm}
\begin{tabular}{lcc}
\toprule
\textbf{Method} & \textbf{Accuracy (\%)} & \textbf{Time (ms)} \\
\midrule
1-WL & 0.0 & 2.58\\
SPD-WL & 20.8 & 7.08\\
ISP-WL$_{\text{Deg}}$ & 23.3 & 4.35\\
ISP-WL$_{\text{BC}}$ & 30.5 & 7.59\\
ISP-WL$_{\text{ANC}}$ & 44.5 & 15.25\\
3-WL & 67.5 & 101.42\\
\bottomrule
\end{tabular}
\end{minipage}
&
\begin{minipage}{0.48\linewidth}
\centering
\subcaption{Regular Graph Performance}
\label{fig:regular}
\vspace{2mm}
\includegraphics[width=0.75\linewidth]{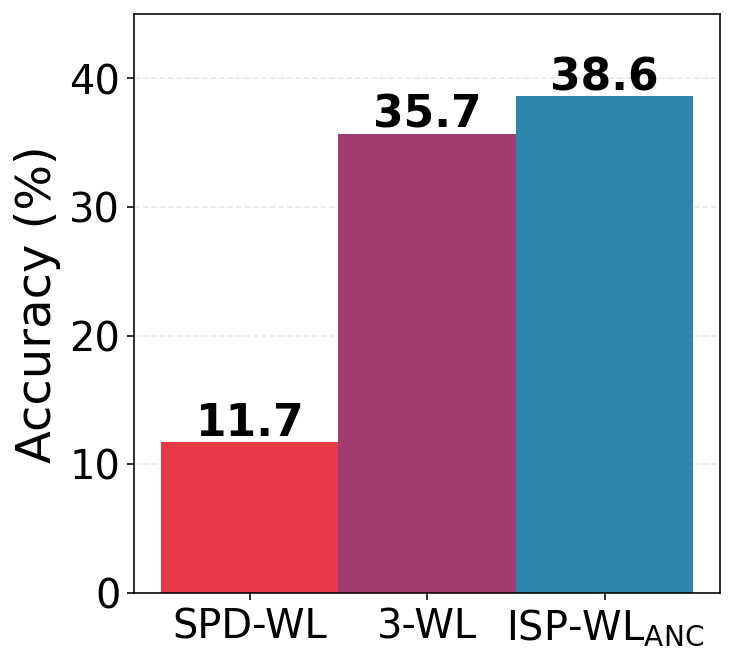}
\end{minipage}
\end{tabular}
\end{table}

\subsubsection{Additional Experiments}In the appendix, we present the following experiments: (1) Performance comparison of different graph invariants (Section \ref{sec:comp_inv}), (2) Ablation studies on Hierarchical Structural Gap Encoding, which include component analysis and extensions to higher-order motifs (Section \ref{sec:hie_gap}), (3) Task-Adaptive Structural Discovery and hyperparameter sensitivity of ISP-GNN$_\dagger$ (Section \ref{sec:learnable_inv}), (4) Scalability analysis of ISP-WL with diverse invariants (Section \ref{sec:scalability}), and (5) Visualization analysis on coloring refinement across invariants, invariant orthogonality, and convergence dynamics (Section \ref{sec:visualization}).

%% file: sections/conclusion.tex
\section{Conclusion, Limitations, and Future Work}

 We introduce Invariant-Stratified Propagation, a framework that transcends the 1-WL expressivity barrier by stratifying nodes hierarchically based on graph invariants. By encoding structural heterogeneity within higher-order patterns, our approach captures differences in how nodes occupy structural positions within local patterns. The framework comprises both a novel WL variant with formal expressivity guarantees and its efficient neural implementation with learnable invariants. Extensive experiments demonstrate consistent improvements over both standard message-passing architectures and state-of-the-art expressive baselines across diverse downstream tasks.

 An in-depth discussion of the limitations and potential future directions of our work is provided in Appendix section \ref{sec:limitations}. Additionally, for our learnable ISP-GNN variant, the number of hierarchical strata and the temperature annealing parameter are set as hyperparameters. While we provide sensitivity analysis to guide practitioners in selecting these values, we plan to learn them in a data-driven manner in our future work.

%% file: sections/appendix.tex
\section*{Appendix}

\section{Background} \label{sec:bac}

\subsection{Graph Neural Networks}
GNNs learn node representations through iterative neighborhood aggregation. At layer $k$, each node $v$ updates its embedding $h_v^{(k)}$ by combining its previous state with aggregated neighbor information:
\begin{equation}
h_v^{(k)} = \text{UPD}^{(k)}\left(h_v^{(k-1)}, \text{AGG}^{(k)}\left(\{h_u^{(k-1)} : u \in N(v)\}\right)\right)
\end{equation}
where $\text{AGG}(\cdot)$ is a permutation-invariant function that collects information from neighbors, and $\text{UPD}(\cdot)$ is a learnable transformation that produces the new node embedding \cite{xu2018powerful}. This message-passing framework maintains permutation equivariance and scales to graphs of varying sizes.

\subsection{Graph Isomorphism}

\begin{definition}[Graph Isomorphism]
Two graphs $G_1 = (V_1, E_1)$ and $G_2 = (V_2, E_2)$ are \textit{isomorphic}, denoted $G_1 \cong G_2$, if there exists a bijection $\pi : V_1 \to V_2$ such that:
\begin{equation}
(u, v) \in E_1 \iff (\pi(u), \pi(v)) \in E_2 \quad \forall u, v \in V_1
\end{equation}
The bijection $\pi$ is called an \textit{isomorphism}.
\end{definition}

Intuitively, isomorphic graphs have identical structure under node relabeling \cite{kobler2012graph}.

\subsection{Weisfeiler-Leman Test}
The 1-WL test \cite{weisfeiler1968reduction} provides a heuristic for graph isomorphism by iteratively refining node colors. Starting with uniform coloring $c^{(0)}(v) = c_0$ for all $v \in V$, each iteration updates colors based on neighbor color multisets: 
\begin{equation}
c^{(k)}(v) = \text{Hash}\left(c^{(k-1)}(v), \{\!\!\{c^{(k-1)}(u) : u \in N(v)\}\!\!\}\right)
\end{equation}
Two graphs are distinguished if their final color multisets differ. The 1-WL test bounds standard MPNN expressivity: with injective aggregation and uniform initialization, GNNs cannot distinguish graphs that 1-WL cannot distinguish \cite{xu2018powerful,morris2019weisfeiler}.

\subsection{Graph Invariants}

\begin{definition}[Graph Invariant]
A graph invariant is a property that is preserved under graph isomorphisms. For node-level invariants, this means a function $\mu: V \to \mathbb{R}$ such that for any graph isomorphism $\pi: V \to V$, we have $\mu(\pi(v)) = \mu(v)$ for all $v \in V$.
\end{definition}

This ensures that the invariant respects the graph's intrinsic structure: nodes with identical structural roles receive identical values, independent of node labeling. Such standard invariants include node degree, core number, onion number, clustering coefficient, PageRank, and betweenness centrality \cite{malliaros2020core,hebert2016multi,zhao2018ranking}. 

\section{Correspondence Between ISP-WL and ISP-GNN} \label{sec:correspondence}

We establish that ISP-GNN faithfully implements the combinatorial principles of ISP-WL. The neural architecture directly mirrors the discrete algorithm's dual-stream structure, with each ISP-WL operation having a continuous differentiable analog in ISP-GNN:

We formalize this correspondence through the following expressivity result:

\begin{theorem}[ISP-WL $\subseteq$ ISP-GNN Expressivity]
\label{thm:correspondence}
If ISP-GNN uses injective aggregation and update functions with sufficient embedding dimension, then ISP-WL distinctions are preserved:
\[
\text{ISP-WL}_K(G_1) \neq \text{ISP-WL}_K(G_2) 
\]
implies there exist parameters $\theta^*$ such that
\[
\text{ISP-GNN}_K(G_1; \theta^*) \neq \text{ISP-GNN}_K(G_2; \theta^*)
\]
\end{theorem}

\begin{proof}
We construct an injective mapping $\Psi^{(k)}: \mathcal{C}^{(k)} \to \mathbb{R}^d$ from ISP-WL colors to ISP-GNN embeddings by induction. At iteration zero, we map initial colors to input features. For the inductive step, assume $\Psi^{(k-1)}$ is injective. Since ISP-WL updates colors by hashing the current color, neighbor colors, and structural information, and ISP-GNN updates embeddings through injective aggregation and transformation of these same components, different colors produce different embeddings. By the inductive hypothesis, different neighbor colors yield different neighbor embeddings, which are aggregated injectively. With sufficient capacity, the update function implements an injective mapping on the finite color space, ensuring $\Psi^{(k)}$ is injective. At the graph level, distinct color multisets map to distinct embedding multisets under $\Psi^{(K)}$, which are distinguished by injective graph-level aggregation.
\end{proof}

This result establishes ISP-GNN as at least as expressive as ISP-WL, validating the neural architecture as a principled implementation of the discrete algorithm while enabling end-to-end learning.

\begin{table}[htbp]
\centering
\caption{Correspondence between ISP-WL algorithm and ISP-GNN neural implementation.}
\label{tab:isp_correspondence}
\begin{tabular}{ll}
\toprule
\textbf{ISP-WL Component} & \textbf{ISP-GNN Component} \\
\midrule
Hash-based color refinement & Neural aggregation + MLP \\
Multiset $\{\!\!\{c(u): u\in N(v)\}\!\!\}$ & $\text{AGG}\{h_u: u\in N(v)\}$ \\
Gap encoding $\delta(v,u,w)$ & Gap features $d_{vuw}$ \\
Layer assignment at $\phi(v)$ & Gating $\gamma^{(k)}_v = \mathbb{I}[\phi(v)=k]$ \\
Color space $\mathcal{C}^{(k)}$ & Embedding space $\mathbb{R}^d$ \\
\bottomrule
\end{tabular}
\end{table}

\section{Design Rationale and Extensions of Hierarchical Gap} \label{sec:hie_gap}

In this section, we discuss the design rationale for our hierarchical gap-encoding mechanism and its possible extensions beyond triangles.

\subsection{Hierarchical Gap Encoding Properties}

The three-component gap encoding $\delta(v,u,w) = (\delta_1, \delta_2, \delta_3)$ is designed to capture the complete structural heterogeneity of triangles through three complementary aspects:

\paragraph{\textbf{Hierarchical Dominance ($\delta_1$):}} The largest gap $\delta_1 = \max(\phi(v) - \phi(u), \phi(v) - \phi(w))$ measures how far the center node is ahead of its most distant neighbor in the structural ordering. Large positive values indicate hub-like roles, while negative values suggest peripheral positioning.

\paragraph{\textbf{Hierarchical Breadth ($\delta_2$):}} The smallest gap $\delta_2 = \min(\phi(v) - \phi(u), \phi(v) - \phi(w))$ measures the gap to the closest neighbor. This indicates whether the center is far from both neighbors (large $|\delta_2|$) or close to at least one neighbor (small $|\delta_2|$).

\paragraph{\textbf{Neighbor Homogeneity ($\delta_3$):}} The inter-neighbor gap $\delta_3 = |\phi(u) - \phi(w)|$ measures structural similarity between neighbors, independent of the center's position. Large values indicate the center connects structurally dissimilar regions, while small values suggest local community structure.

These three components form a minimal sufficient statistic: given $\phi(v)$ and $\delta$, the multiset $\{\phi(v), \phi(u), \phi(w)\}$ can be recovered, yet no component is redundant. Moreover, by Proposition~\ref{prop:permutationinv}, the encoding satisfies $\delta(v,u,w) = \delta(v,w,u)$, enabling efficient hash-based coloring in ISP-WL and consistent neural processing.

\begin{figure}[htbp]
\centering
\begin{subfigure}[b]{0.7\columnwidth}
    \centering
    \includegraphics[width=\textwidth]{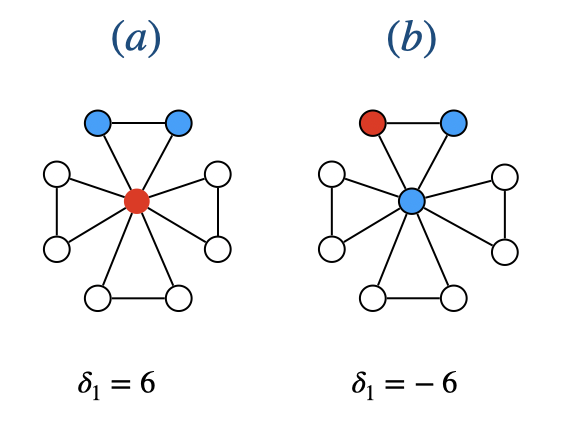}
    \caption{Hierarchical dominance ($\delta_1$)}
    \label{fig:hierarchical_dominance}
\end{subfigure}

\vspace{0.5cm}

\begin{subfigure}[b]{0.9\columnwidth}
    \centering
    \includegraphics[width=\textwidth]{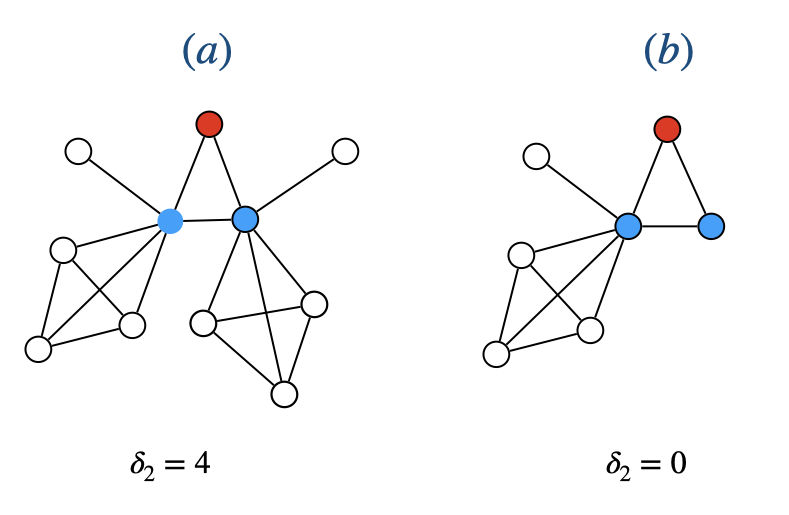}
    \caption{Hierarchical breadth ($\delta_2$)}
    \label{fig:hierarchical_breadth}
\end{subfigure}

\vspace{0.5cm}

\begin{subfigure}[b]{0.7\columnwidth}
    \centering
    \includegraphics[width=\textwidth]{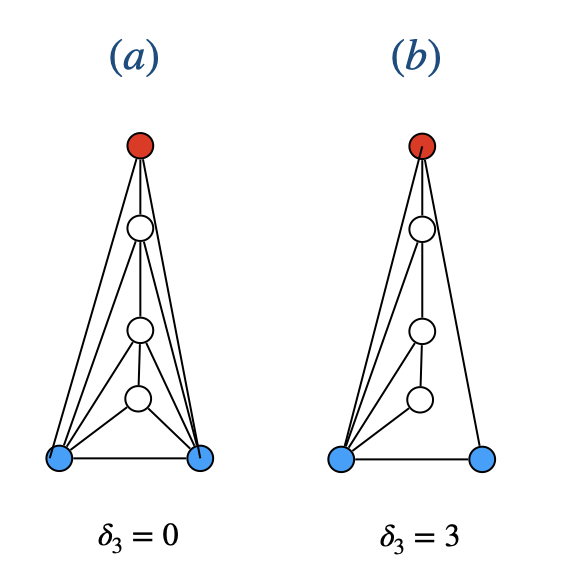}
    \caption{Neighbor homogeneity ($\delta_3$)}
    \label{fig:neighbor_homogeneity}
\end{subfigure}

\caption{Gap encoding components capture complementary structural properties using degree as invariant. (a) $\delta_1$ differentiates hub vs. peripheral roles. (b) $\delta_2$ captures structural proximity to the closest neighbour. (c) $\delta_3$ measures neighbor similarity.}
\label{fig:gap_components}
\end{figure}

\subsection{Extension to Higher-Order Motifs}

The gap encoding can be generalised from triangles to arbitrary higher-order patterns. For any motif with a designated center node $v$ and participating neighbors $\{u_1, \ldots, u_{m}\}$, we compute:
\begin{equation}
\delta^{\text{center}}(v) = \text{sort}([\phi(v) - \phi(u_i)]_{i=1}^{m}) \in \mathbb{R}^{m}
\label{eq:center_gaps}
\end{equation}
which captures sorted center-to-neighbor gaps, and
\begin{equation}
\delta^{\text{inter}} = \text{sort}([|\phi(u_i) - \phi(u_j)|]_{(i,j) \in E_M}) \in \mathbb{R}^{|E_M|}
\label{eq:inter_gaps}
\end{equation}
which captures gaps between connected neighbors in the motif, where $E_M$ denotes edges among neighbors within the motif. Sorting ensures permutation invariance.

For triangles, this reduces to our original encoding with $\delta^{\text{center}} = (\delta_1, \delta_2)$ and $\delta^{\text{inter}} = (\delta_3)$. For $k$-cliques, where all neighbors are connected, $|E_M| = \binom{k-1}{2}$ and we obtain $\frac{(k-1)k}{2}$ total values. For other motifs, such as paths or stars, the number of inter-neighbor gaps depends on the motif's edge structure.

While this framework extends to arbitrary patterns, we focus on triangles for computational efficiency.

\section{Graph Invariants} \label{sec:invariants}

ISP-WL supports any graph invariant that satisfies permutation equivariance. Table~\ref{tab:invariants} summarizes the invariants used in our experiments along with their computational complexity.

\begin{table*}[h]
\centering
\caption{Graph invariants used in our experiments and their computational complexities on sparse graphs with bounded degeneracy.}
\label{tab:invariants}
\begin{tabular}{lll}
\toprule
\textbf{Invariant} & \textbf{Complexity} & \textbf{Description} \\
\midrule
Degree & $O(|V|)$ & Number of neighbors \\
K-Core \cite{malliaros2020core} & $O(|V| + |E|)$ & Maximal core membership \\
Onion \cite{hebert2016multi} & $O(|V| + |E|)$ & Hierarchical k-core layers \\
Clustering Coefficient \cite{li2017clustering} & $O(d|E|)$ & Local triangle density ratio \\
Average Neighborhood Clustering \cite{wilson2018calculating} & $O(d^2|E|)$ & Mean clustering coefficient of neighbors \\
K-Truss \cite{wu2018k} & $O(|E|^{1.5})$ & Triangle-support decomposition \\
PageRank \cite{kim2002improved} & $O(k(|V| + |E|))$ & Stationary distribution ($k$ iterations) \\
Eigenvector \cite{bonacich1987power} & $O(k(|V| + |E|))$ & Principal eigenvector ($k$ iterations) \\
Betweenness \cite{barthelemy2004betweenness} & $O(|V||E|)$ & Shortest path centrality \\
\bottomrule
\end{tabular}
\end{table*}

\section{Proofs}

In this section, we provide detailed proofs for all theoretical results presented in the main content of the paper.

\subsection{Methodology Section}

\permutationinv*

\begin{proof}
By definition, $\delta(v, u, w) = (\delta_1, \delta_2, \delta_3)$ where:
\begin{align*}
\delta_1 &= \max(\phi(v) - \phi(u), \phi(v) - \phi(w)) \\
\delta_2 &= \min(\phi(v) - \phi(u), \phi(v) - \phi(w)) \\
\delta_3 &= |\phi(u) - \phi(w)|
\end{align*}

Since $\max$ and $\min$ are symmetric functions and absolute value is symmetric, we have:
\begin{align*}
\delta_1(v, u, w) &= \max(\phi(v) - \phi(u), \phi(v) - \phi(w)) \\
&= \max(\phi(v) - \phi(w), \phi(v) - \phi(u)) \\
&= \delta_1(v, w, u) \\
\delta_2(v, u, w) &= \min(\phi(v) - \phi(u), \phi(v) - \phi(w)) \\
&= \min(\phi(v) - \phi(w), \phi(v) - \phi(u)) \\
&= \delta_2(v, w, u) \\
\delta_3(v, u, w) &= |\phi(u) - \phi(w)| \\
&= |\phi(w) - \phi(u)| = \delta_3(v, w, u)
\end{align*}

Therefore, $\delta(v, u, w) = \delta(v, w, u)$.
\end{proof}

\singleassign*

\begin{proof}
At iteration $\ell$, the algorithm processes nodes in $V_\ell = \{v \in V : \phi(v) = \ell\}$ where $\ell \in \{1, 2, \ldots, L\}$. For node $v$ with invariant value $\phi(v) \in \{1, \ldots, L\}$, we have $v \in V_\ell$ if and only if:
\begin{equation*}
\phi(v) = \ell
\end{equation*}

Since $\phi(v)$ is fixed and takes exactly one value from $\{1, \ldots, L\}$, there exists exactly one iteration $\ell^* = \phi(v)$ where $v \in V_{\ell^*}$.

At iteration $\ell^*$, the ISP color is assigned:
\begin{equation*}
c_{\text{ISP}}(v) \gets \begin{cases}
\text{Hash}(\phi(v), M^{\text{struct}}_v) & \text{if } M^{\text{struct}}_v \neq \{\!\!\{\}\!\!\} \\
\phi(v) & \text{otherwise}
\end{cases}
\end{equation*}

For all $\ell \neq \ell^*$, we have $\phi(v) \neq \ell$, thus $v \notin V_\ell$ and $c_{\text{ISP}}(v)$ is not modified.
\end{proof}

\structcons*

\begin{proof}
For any automorphism $\pi : V \to V$, which maps structurally equivalent nodes to each other:
\begin{align*}
\phi_{\text{learn}}(\pi(v)) &= L \cdot \sigma(\text{MLP}_{\phi}([\phi_1(\pi(v)), \ldots, \phi_m(\pi(v))])) \\
&= L \cdot \sigma(\text{MLP}_{\phi}([\phi_1(v), \ldots, \phi_m(v)])) \\
&= \phi_{\text{learn}}(v)
\end{align*}
where the second equality holds because each $\phi_i$ is a graph invariant. Since the floor operation preserves equality, $\phi_{\text{discrete}}(\pi(v)) = \phi_{\text{discrete}}(v)$ as well.
\end{proof}

\subsection{Expressivity}

\beyondwl*

\begin{proof}
First, we show that if 1-WL distinguishes any graph pair $G_1$ and $G_2$, then ISP-WL also distinguishes them. 

Recall that ISP-WL maintains a WL stream that performs standard 1-WL refinement at each iteration. Specifically, for each node $v$, the WL color is updated as:
\[
c'_{\mathit{WL}}(v) \gets \mathrm{Hash}((c_{\mathit{WL}}(v), c_{\mathit{ISP}}(v)), m_v)
\]
where $m_v = \{(c_{\mathit{WL}}(u), c_{\mathit{ISP}}(u)) : u \in N(v)\}$ is the multiset of neighbor colors. When the ISP stream is disabled (i.e., $c_{\mathit{ISP}}(v) = \bot$ for all $v$), the WL refinement reduces to standard 1-WL. Since ISP-WL includes the WL stream and only adds additional information through the ISP stream, if $c_{\mathit{1\text{-}WL}}(G_1) \neq c_{\mathit{1\text{-}WL}}(G_2)$, then the WL stream in ISP-WL will distinguish them at some iteration $k$.
The final ISP-WL coloring is $c(v) = \mathrm{Hash}(c_{\mathit{WL}}(v), c_{\mathit{ISP}}(v))$, which incorporates $c_{\mathit{WL}}(v)$.
Therefore, if the WL colors differ, the final ISP-WL colors must differ. Thus, ISP-WL is at least as expressive as 1-WL.

To establish strictly higher expressivity, it suffices to show at least one non-isomorphic graph pair that 1-WL cannot distinguish but ISP-WL can. Figure~\ref{fig:separation-examples} provides three such examples. Therefore, ISP-WL is strictly more expressive than 1-WL.

\end{proof}

\nodedist*

\begin{proof}We first prove that for any nodes $u, v$, if $c_{\text{1-WL}}(u) \neq c_{\text{1-WL}}(v)$, then $c_{\text{ISP-WL}}(u) \neq c_{\text{ISP-WL}}(v)$.  ISP-WL maintains a WL stream that performs standard 1-WL refinement at each iteration:
\begin{align*}
c^{(k)}_{\text{WL}}(v) = \text{Hash}\big(&(c^{(k-1)}_{\text{WL}}(v), c^{(k-1)}_{\text{ISP}}(v)), \\
&\{\!\!\{(c^{(k-1)}_{\text{WL}}(u), c^{(k-1)}_{\text{ISP}}(u)) : u \in N(v)\}\!\!\}\big)
\end{align*}

When ISP colors are ignored (all $c_{\text{ISP}} = \bot$), this update rule reduces to standard 1-WL refinement. Since ISP-WL only adds information via the ISP stream and uses injective hashing, any distinction made by standard 1-WL is preserved in the WL stream of ISP-WL. More precisely, if standard 1-WL distinguishes nodes $u$ and $v$ at iteration $k$, meaning $c^{(k)}_{\text{1-WL}}(u) \neq c^{(k)}_{\text{1-WL}}(v)$, then their WL colors in ISP-WL must also differ: $c^{(k)}_{\text{WL}}(u) \neq c^{(k)}_{\text{WL}}(v)$.  The final ISP-WL coloring combines both streams via $c_{\text{ISP-WL}}(v) = \text{Hash}(c_{\text{WL}}(v), c_{\text{ISP}}(v))$. By the injectivity of the hash function, if the WL components differ, the final colorings must differ. Therefore, we conclude that $c_{\text{1-WL}}(u) \neq c_{\text{1-WL}}(v)$ implies $c_{\text{ISP-WL}}(u) \neq c_{\text{ISP-WL}}(v)$.

To prove ISP-WL can distinguish nodes even when $c_{\text{1-WL}}(u) = c_{\text{1-WL}}(v)$, it is sufficient to provide a graph pair that contains two such nodes. Consider graphs $G_1$ and $G_2$ in Figure~\ref{fig:graph-pair} with identical degree sequences. Let $u \in G_1$ and $v \in G_2$ denote any nodes. Both have degree 2 with identical neighborhood degree sequences, thus $c_{\text{1-WL}}(u) = c_{\text{1-WL}}(v)$.

\begin{figure}[h]
\centering
\includegraphics[width=0.2\textwidth]{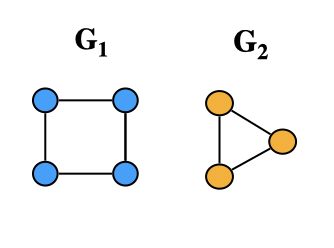}
\caption{Graphs $G_1$ and $G_2$ with identical degree sequences but different triangle structures.}
\label{fig:graph-pair}
\end{figure}

However, $|T(u)| = 0 \neq 1 = |T(v)|$. ISP-WL computes:
$$M^{\text{struct}}_{u} = \{\!\!\{(c_{\text{ISP}}(u'), c_{\text{ISP}}(u''), \delta(u, u', u'')) : (u',u'') \in T(u)\}\!\!\}$$
Since $|M^{\text{struct}}_{u}| \neq |M^{\text{struct}}_{v}|$, by injective hashing:
$$c_{\text{ISP}}(u) \neq c_{\text{ISP}}(v)$$
Therefore, $c_{\text{ISP-WL}}(u) \neq c_{\text{ISP-WL}}(v)$ despite $c_{\text{1-WL}}(u) = c_{\text{1-WL}}(v)$.
\ 
\end{proof}

\invexpress*

\begin{proof}
We first prove statement 1. Theorem~\ref{thm:beyond-1wl} shows that ISP-WL strictly extends 1-WL expressivity, giving us $\text{1-WL} \prec \text{ISP-WL}(\phi)$ for any $\phi$. For the upper bound, if $\phi \preceq \text{1-WL}$, it can be computed by 1-WL coloring. ISP-WL uses information from 1-WL colors and triangle aggregation. Since 3-WL distinguishes graphs based on all triplets \cite{chen2025enhanced}, and ISP-WL operates on local triplets (neighborhoods) rather than globally, its expressivity is bounded by 3-WL.

Then, we prove statement 2. By definition of incomparability, $\phi$ distinguishes some graph pairs that $k$-WL cannot, and vice versa. We show ISP-WL($\phi$) inherits both capabilities. For the first direction, if $\phi$ distinguishes graphs $G_1, G_2$ but $k$-WL does not, then their $\phi$-value multisets differ: $\{\!\!\{\phi(u) : u \in V(G_1)\}\!\!\} \neq \{\!\!\{\phi(v) : v \in V(G_2)\}\!\!\}$. Since ISP-WL incorporates $\phi$ values through stratification and the ISP stream assigns $c_{\text{ISP}}(u) = \phi(u)$ for nodes without triangles, the final color multisets differ: $\{\!\!\{c_{\text{ISP-WL}}(u) : u \in V(G_1)\}\!\!\} \neq \{\!\!\{c_{\text{ISP-WL}}(v) : v \in V(G_2)\}\!\!\}$. For the second direction, if $k$-WL distinguishes graphs $H_1, H_2$ but $\phi$ does not, then their $\phi$-value multisets are identical. ISP-WL's expressivity is constrained by its WL stream (bounded by 1-WL, which is weaker than $k$-WL for $k > 1$) and its ISP stream (derived from $\phi$ and triangle aggregation). Since $\phi$ cannot distinguish $H_1, H_2$, the stratification and gap encodings are identical, and the WL stream alone cannot compensate for higher-order patterns that $k$-WL captures through tuple-level coloring. Therefore, ISP-WL($\phi$) and $k$-WL are incomparable.

For statement 3, assume $k$-WL $\preceq \phi$, meaning any graphs distinguished by $k$-WL are also distinguished by $\phi$. We show that ISP-WL($\phi$) preserves all distinctions made by $\phi$, and hence by $k$-WL. Consider any non-isomorphic graphs $G_1 \not\cong G_2$ that $k$-WL distinguishes. By assumption, $\phi$ also distinguishes them: $\{\!\!\{\phi(u) : u \in V(G_1)\}\!\!\} \neq \{\!\!\{\phi(v) : v \in V(G_2)\}\!\!\}$. The ISP stream directly incorporates $\phi$ values through two mechanisms: (1) stratification into layers $\ell \in \{1, \ldots, L\}$ based on $\phi$-value rankings, and (2) base ISP colors where $c_{\text{ISP}}(u) = \phi(u)$ for nodes without triangles. For nodes with triangles, ISP colors are computed as $c_{\text{ISP}}(u) = \text{Hash}(\phi(u), M^{\text{struct}}_u)$, which still incorporates $\phi(u)$ directly. By injectivity of the hash function, if $\phi(u) \neq \phi(v)$, then $c_{\text{ISP}}(u) \neq c_{\text{ISP}}(v)$ regardless of structural messages. Therefore, different $\phi$ multisets lead to different ISP color multisets, which are preserved in the final ISP-WL coloring: $\{\!\!\{c_{\text{ISP-WL}}(u) : u \in V(G_1)\}\!\!\} = \{\!\!\{\text{Hash}(c_{\text{WL}}(u), c_{\text{ISP}}(u)) : u \in V(G_1)\}\!\!\} \neq \{\!\!\{\text{Hash}(c_{\text{WL}}(v), c_{\text{ISP}}(v)) : v \in V(G_2)\}\!\!\}$. Thus, any pair distinguished by $k$-WL is distinguished by $\phi$, which in turn is distinguished by ISP-WL($\phi$), giving us $k\text{-WL} \preceq \text{ISP-WL}(\phi)$. The proof is complete.
   
\end{proof}

\subsection{Convergence and Complexity}

\convergence*

\begin{proof}
The algorithm maintains two streams: the WL stream and the ISP stream. The WL refinement operates on a finite color space bounded by $|V|$. The refinement is monotonic: distinct colors remain distinct. By standard 1-WL convergence results, this reaches a fixed point in at most $K_{\text{WL}} \leq |V|$ iterations. In ISP Stream, since $\phi(v) \in \{1, 2, \ldots, L\}$ for all $v \in V$, there are exactly $L$ distinct invariant layers. At iteration $\ell \in \{1, \ldots, L\}$, the algorithm processes nodes in $V_\ell = \{v \in V : \phi(v) = \ell\}$. By the Single Assignment Property, each node is processed exactly once when its layer $\ell = \phi(v)$ is reached. The ISP stream completes after processing all $L$ layers.

Therefore, the algorithm terminates in $K = \max(K_{\text{WL}}, L) \leq |V|$ since $L \leq |V|$.
\end{proof}

\complexity*

\begin{proof}
Computing $\phi(v)$ for all nodes takes $O(|V|)$ for simple invariants like degree, or $O(|E| $ for invariants like k-core and onion.

For each node $v$, collecting the neighbor multiset requires $O(d_v)$ time. Summing over all nodes: $\sum_{v \in V} O(d_v) = O(|E|)$. Computing the hash takes $O(1)$ per node, giving $O(|V|)$ total. Thus, WL refinement per iteration takes $O(|V| + |E|)$, and over $K$ iterations: $O(K(|V| + |E|))$.

By the Single Assignment Property, each node is processed once. For node $v$, enumerating triangles $T(v)$ requires checking all neighbor pairs: $O(d_v^2)$. Computing hierarchical gap encodings and hashing for $|T(v)|$ triangles takes $O(|T(v)|)$. Total per node: $O(d_v^2 + |T(v)|) = O(d_v^2)$ since $|T(v)| \leq \binom{d_v}{2}$. Summing over all nodes, the total triangle enumeration cost is $O(T)$ since each triangle is enumerated three times (once per node).

The total complexity is $O(|E| \log |V|) + O(K(|V| + |E|)) + O(T) = O(K(|V| + |E|) + T)$, where preprocessing is dominated by iteration costs for $K \geq 1$.

For graphs with degeneracy $d$, we have $T \leq d \cdot |E|$. When $d = O(1)$, we can conclude that $O(K(|V| + |E|) + T) = O(K(|V| + |E|) + d|E|) = O(K(|V| + |E|))$.
\end{proof}

\subsection{Oversmoothing Resistance}

\oversmoothing*

\begin{proof}
The persistence $h_v^{(k),\text{ISP}} = \bar{h}_v^{\text{ISP}}$ for $k \geq \phi(v)$ follows from the gating mechanism:
\[
\gamma_v^{(k)} = \mathbb{I}[\phi(v) = k] \cdot \mathbb{I}[\|h_v^{(k-1),\text{ISP}}\|_2 = 0]
\]
For $k = \phi(v)$, both indicators are true, so $\gamma_v^{(k)} = 1$ and $h_v^{(k),\text{ISP}}$ is assigned. For $k > \phi(v)$, the first indicator is false, yielding $\gamma_v^{(k)} = 0$ and $h_v^{(k),\text{ISP}} = h_v^{(k-1),\text{ISP}} = \bar{h}_v^{\text{ISP}}$ by induction.

Since $\bar{h}_u^{\text{ISP}}$ and $\bar{h}_v^{\text{ISP}}$ are depth-independent and distinct, the concatenated inputs differ:
\[
[h_u^{(K),\text{WL}} \| \bar{h}_u^{\text{ISP}}] - [h_v^{(K),\text{WL}} \| \bar{h}_v^{\text{ISP}}] = [h_u^{(K),\text{WL}} - h_v^{(K),\text{WL}} \| \bar{h}_u^{\text{ISP}} - \bar{h}_v^{\text{ISP}}]
\]
The ISP component $\bar{h}_u^{\text{ISP}} - \bar{h}_v^{\text{ISP}} \neq 0$ ensures the inputs remain distinct regardless of WL stream behavior. If $\text{MLP}_{\text{out}}$ is injective, distinct inputs produce distinct outputs.
\end{proof}

The depth-invariant structural information $\bar{h}_v^{\text{ISP}}$ acts as an anchor, preventing feature collapse in deep networks while the WL stream continues refining higher-order patterns.
\section{Experimental Design} \label{sec:exp_des}

\subsection{Downstream Tasks}

We evaluate ISP-GNN across three fundamental graph learning tasks that assess different aspects of the model's expressive power and generalization capabilities:

\subsubsection{Node Classification} The goal is to predict labels for individual nodes using both node features and graph structure. 

\subsubsection{Influence Estimation} Formulated as node-level regression, this task predicts the activation probability of each node under information diffusion models (IC, LT, and SIS). 

\subsubsection{Graph Classification} The objective is to predict a single label for an entire graph based on its structure and node features. 

\subsection{Dataset Statistics}

We conduct a comprehensive evaluation of ISP-GNN across 26 benchmark datasets covering three graph learning tasks: node classification, influence estimation, and graph classification. 

\subsubsection{Node Classification}

We employ 10 benchmark datasets, as shown in Table~\ref{tab:dataset_stats_node}, that span both homophilic and heterophilic network structures. The homophilic datasets include Cora ML, Citeseer, Pubmed, DBLP, and Amazon Photo, which exhibit strong assortativity where connected nodes typically share similar labels. In contrast, the heterophilic datasets comprise Cornell, Wisconsin, Texas, Film, and Amazon-Ratings, which feature disassortative mixing where neighboring nodes often belong to different classes.

\subsubsection{Influence Estimation}

We utilize 4 real-world social networks, listed in Table~\ref{tab:dataset_stats_influence}, which exhibit diverse topological properties ranging from small-world to scale-free characteristics. These networks include Jazz, Cora-ML, Network Science, and Power Grid. 

\subsubsection{Graph Classification}

We evaluate 12 datasets from the TU benchmark and Open Graph Benchmark, covering various domains and scales. The benchmarks include: (1) Molecular graphs (MUTAG, PTC-MR, BZR, DHFR, COX2, ogbg-moltox21, ogbg-molpcba) where nodes are atoms, (2) Bioinformatics (PROTEINS, D\&D) with protein interaction networks, and (3) Social networks (IMDB-BINARY, IMDB-MULTI, Reddit-Binary) for collaborations and communities.

\begin{table}[htbp]
\centering
\caption{Dataset statistics for node classification.}
\label{tab:dataset_stats_node}
\small
\begin{tabular}{lcccc}
\toprule
\textbf{Dataset} & \textbf{\# Nodes} & \textbf{\# Edges} & \textbf{\# Features} & \textbf{\# Classes} \\
\midrule
Cora ML & 2,995 & 8,416 & 2,879 & 7 \\
Citeseer & 3,327 & 4,732 & 3,703 & 6 \\
Pubmed & 19,717 & 44,338 & 500 & 3 \\
DBLP & 17,716 & 105,734 & 1,639 & 4 \\
Film & 7,600 & 33,544 & 932 & 5 \\
Amazon Photo & 7,650 & 119,081 & 745 & 8 \\
Amazon-Ratings & 24,492 & 93,050 & 300 & 5 \\
Cornell & 183 & 295 & 1,703 & 5 \\
Wisconsin & 251 & 499 & 1,703 & 5 \\
Texas & 183 & 309 & 1,703 & 5 \\
\bottomrule
\end{tabular}
\end{table}

\begin{table}[htbp]
\centering
\caption{Dataset statistics for influence estimation.}
\label{tab:dataset_stats_influence}
\small
\begin{tabular}{lcc}
\toprule
\textbf{Dataset} & \textbf{\# Nodes} & \textbf{\# Edges} \\
\midrule
Jazz & 198 & 2,742 \\
Cora-ML & 2,810 & 7,981 \\
Network Science & 1,565 & 13,532 \\
Power Grid & 4,941 & 6,594 \\
\bottomrule
\end{tabular}
\end{table}

\begin{table*}[!htbp]
\centering
\caption{Dataset statistics for graph classification.}
\label{tab:dataset_stats_graph}
\small
\begin{tabular}{lccccc}
\toprule
\textbf{Dataset} & \textbf{Domain} & \textbf{\# Graphs} & \textbf{Avg \# Nodes} & \textbf{Avg \# Edges} & \textbf{\# Classes} \\
\midrule
MUTAG & Molecular & 188 & 17.93 & 19.79 & 2 \\
PTC\_MR & Molecular & 344 & 14.29 & 14.69 & 2 \\
BZR & Molecular & 405 & 35.75 & 38.36 & 2 \\
DHFR & Molecular & 467 & 42.43 & 44.54 & 2 \\
COX2 & Molecular & 467 & 41.22 & 43.45 & 2 \\
PROTEINS & Bioinformatics & 1,113 & 39.06 & 72.82 & 2 \\
D\&D & Bioinformatics & 1,178 & 284.32 & 715.66 & 2 \\
IMDB-BINARY & Social & 1,000 & 19.77 & 96.53 & 2 \\
IMDB-MULTI & Social & 1,500 & 13.00 & 65.94 & 3 \\
Reddit-Binary & Social & 2,000 & 429.63 & 497.75 & 2 \\
ogbg-moltox21 & Molecular & 7,831 & 18.51 & 19.26 & 12 \\
ogbg-molpcba & Molecular & 437,929 & 26.0 & 28.1 & 128 \\
\bottomrule
\end{tabular}
\end{table*}

\begin{table*}[t]
\centering
\caption{Node and graph classification performance with different structural invariants. Values show accuracy (\%) with standard deviation. Best results are highlighted in \textbf{bold}.}
\label{tab:performance_invariants}
\begin{tabular}{l|cc|cc}
\toprule
\multirow{2}{*}{\textbf{Method}} & \multicolumn{2}{c|}{\textbf{Node Classification}} & \multicolumn{2}{c}{\textbf{Graph Classification}} \\
& \textbf{Amazon Photo} & \textbf{Amazon-Ratings} & \textbf{ogbg-moltox21} & \textbf{Reddit-Binary} \\
\midrule
Baseline & 89.30 $\pm$ 0.82 & 37.99 $\pm$ 0.61 & 74.9 $\pm$ 0.5 & 92.2 $\pm$ 2.3 \\
Degree & 93.12 $\pm$ 0.46 & 47.5 $\pm$ 1.03 & 75.1 $\pm$ 0.6 & 92.8 $\pm$ 2.2 \\
K-core & 93.97 $\pm$ 0.46 & \textbf{48.43 $\pm$ 1.21} & 75.1 $\pm$ 0.8 & 92.5 $\pm$ 2.1 \\
Onion & 93.93 $\pm$ 0.35 & 44.95 $\pm$ 0.54 & \textbf{75.6 $\pm$ 0.7} & 92.4 $\pm$ 2.7 \\
Clustering Coefficient & 93.23 $\pm$ 0.31 & 46.27 $\pm$ 0.78 & 74.8 $\pm$ 0.6 & 93.0 $\pm$ 2.7 \\
Average Neighborhood Clustering & 93.11 $\pm$ 0.67 & 46.42 $\pm$ 1.00 & 74.9 $\pm$ 1.0 & 92.1 $\pm$ 3.0 \\
K-Truss & \textbf{94.01 $\pm$ 0.65} & 47.83 $\pm$ 1.59 & 75.0 $\pm$ 0.5 & \textbf{93.4 $\pm$ 2.6} \\
\bottomrule
\end{tabular}%
\end{table*}

\begin{figure*}[htbp]
    \centering
    \includegraphics[width=0.95\textwidth]{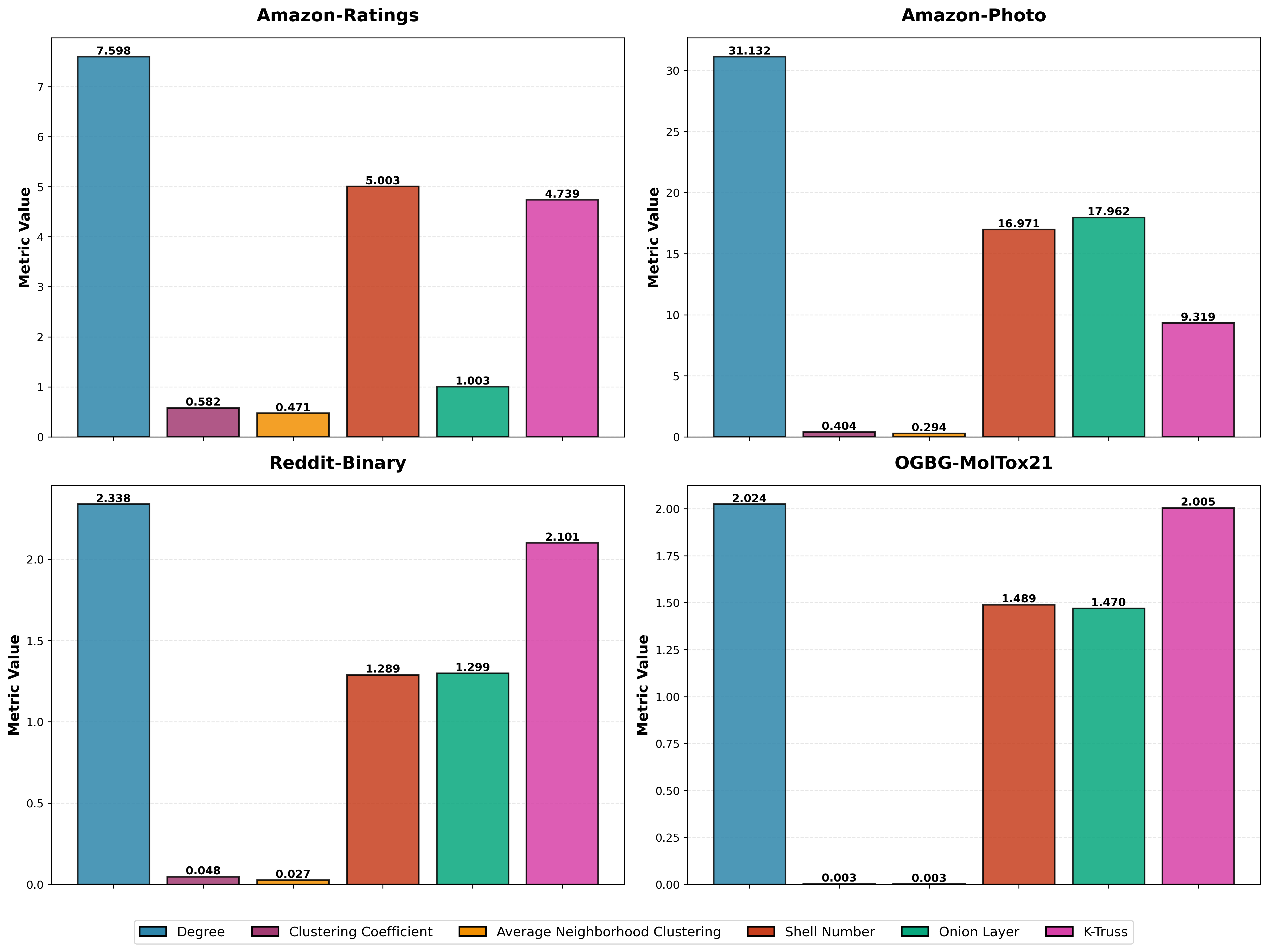}
    \caption{Structural metrics comparison across datasets.}
    \label{fig:structural_metrics}
\end{figure*}

\subsection{Model Hyperparameters}

To ensure optimal model performance across different datasets, we conduct a hyperparameter search. All models are trained using the Adam optimizer \cite{kingma2014adam}. For the OGB datasets, we employ a learning rate scheduler with a decay factor of $0.5$ applied every $10$ epochs. Table~\ref{tab:hyperparameters} summarizes the complete hyperparameter search space used in our experiments.

\begin{table}[h]
\centering
\caption{Hyperparameter search space for ISP-GNN variants}
\label{tab:hyperparameters}
\small
\setlength{\tabcolsep}{4pt}
\begin{tabular}{ll}
\hline
\textbf{Hyperparameter} & \textbf{Search Space} \\
\hline
Number of layers& $\{2, 3, 4, 5, 6\}$ \\
Temperature ($\beta$) & $\{0.5, 1.0, 1.5, 2.0\}$ \\
Number of Strata ($S$) & $\{4, 6, 8\}$ \\
Dropout rate & $\{0.0, 0.1, 0.2, 0.5, 0.6, 0.9\}$ \\
Learning rate  & $\{0.005, 0.009, 0.01, 0.1\}$ \\
Batch size& $\{32, 64, 100\}$ \\
Hidden dimensionality  & $\{32, 64, 128, 256, 300\}$ \\
Weight decay & $\{5\!\times\!10^{-4}, 9\!\times\!10^{-3}, 1\!\times\!10^{-2}, 1\!\times\!10^{-1}\}$ \\
Training epochs& $\{100, 200, 1000\}$ \\
\hline
\end{tabular}
\end{table}

\subsection{Computational Resources}

All experiments were performed on a Linux server with an Intel Xeon W-2175 processor (2.50GHz, 28 cores), NVIDIA RTX A6000 GPU, and 512GB RAM. The software environment specifications are listed in Table~\ref{tab:software}.

\begin{table}[h]
\centering
\caption{Software environment specifications}
\label{tab:software}
\small
\begin{tabular}{ll}
\hline
\textbf{Package} & \textbf{Version} \\
\hline
Python & 3.11.5 \\
PyTorch & 2.3.1 \\
PyTorch Geometric & 2.7.0 \\
torch-cluster & 1.6.3 \\
torch-scatter & 2.0.9 \\
torch-sparse & 0.6.18 \\
\hline
\end{tabular}
\end{table}

\section{Additional Experiments}

\subsection{Performance Comparison with Different Invariants} \label{sec:comp_inv}

Table~\ref{tab:performance_invariants} presents the classification performance of ISP-GNN across node and graph classification tasks using different structural invariants. Figure~\ref{fig:structural_metrics} shows the corresponding structural characteristics of each dataset.

The results demonstrate that ISP-WL provides considerable performance boosts in node classification by incorporating strong structural bias, enhancing baseline performance significantly. For Amazon-Ratings, the graph displays heterophilic characteristics where connected nodes have dissimilar labels. Core/onion-based invariants achieve superior performance as they capture structural positions and hierarchical roles that clustering-based features fail to distinguish. Amazon-Photo, with its dense structure and deep hierarchical organization, also benefits from structural invariants, with k-truss and k core providing notable performance gains. For graph classification, the molecular benchmark ogbg-moltox21 shows less improvements across all invariants, as molecular properties depend on domain-specific chemical semantics rather than graph-theoretic topology. Reddit-Binary, being a social network benchmark with sparse triangular structure, benefits most from k-truss as it captures hierarchical community patterns and cohesive substructure rather than local triangular motifs. Overall, these findings highlight that ISP-GNN achieves optimal performance by selecting invariants that align with the underlying label homophily, structural characteristics, and domain semantics of the target graphs.

\begin{figure*}[H]
    \centering
    \includegraphics[width=0.9\textwidth]{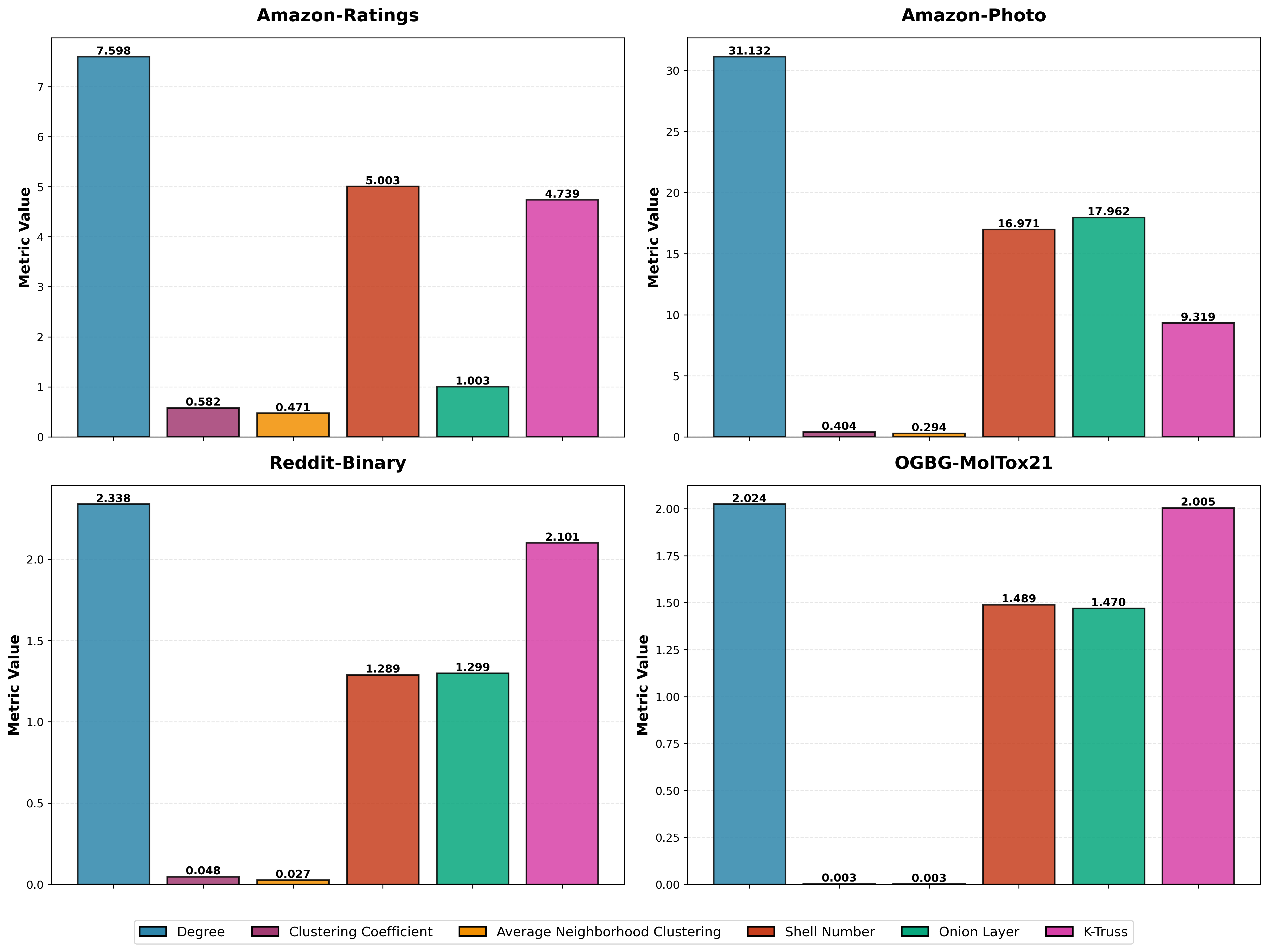}
    \caption{Structural metrics comparison across datasets.}
    \label{fig:structural_metrics}
\end{figure*}

\begin{figure*}[H]
\centering
\begin{tabular}{cc}
\includegraphics[width=0.48\textwidth]{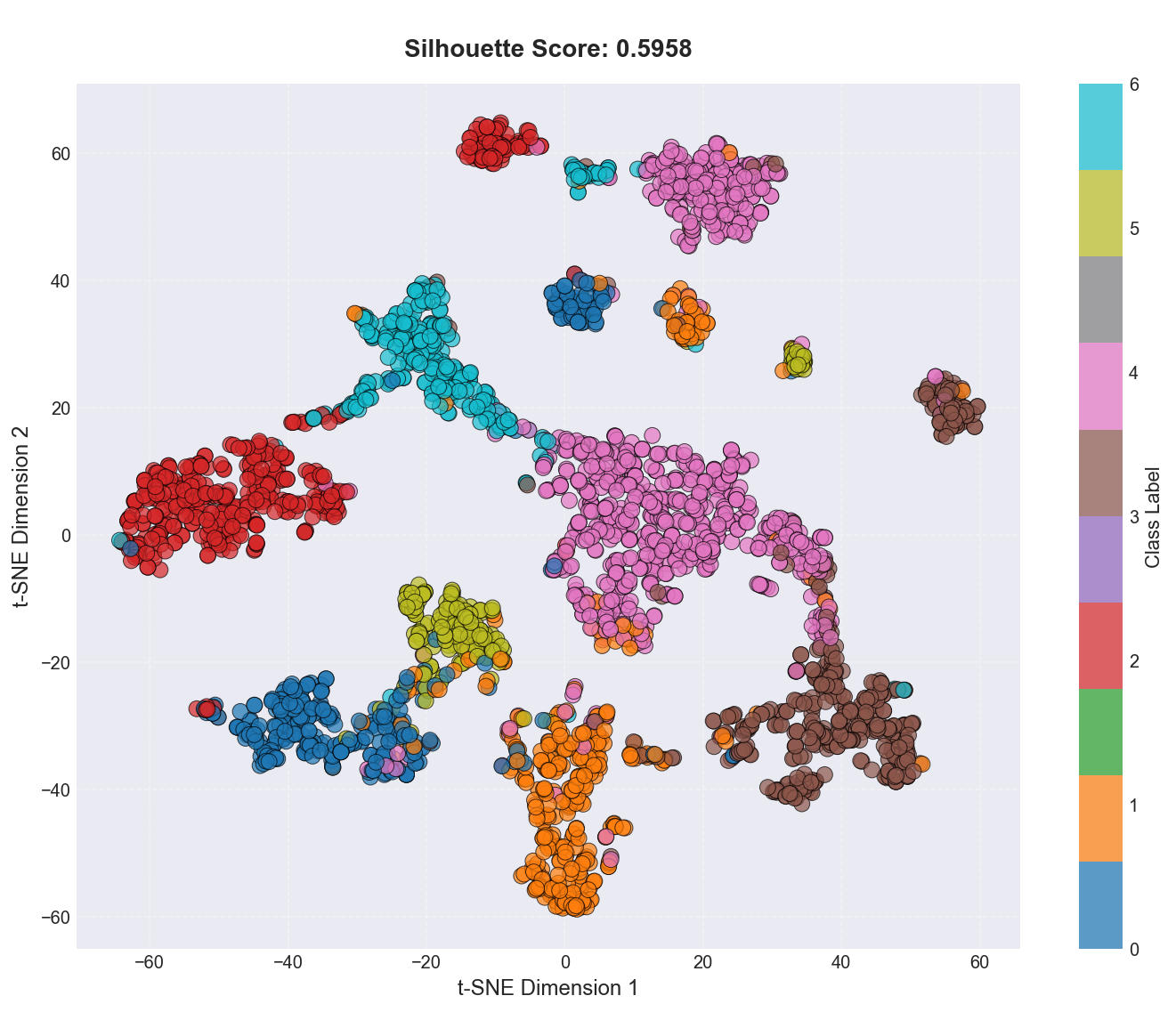} &
\includegraphics[width=0.48\textwidth]{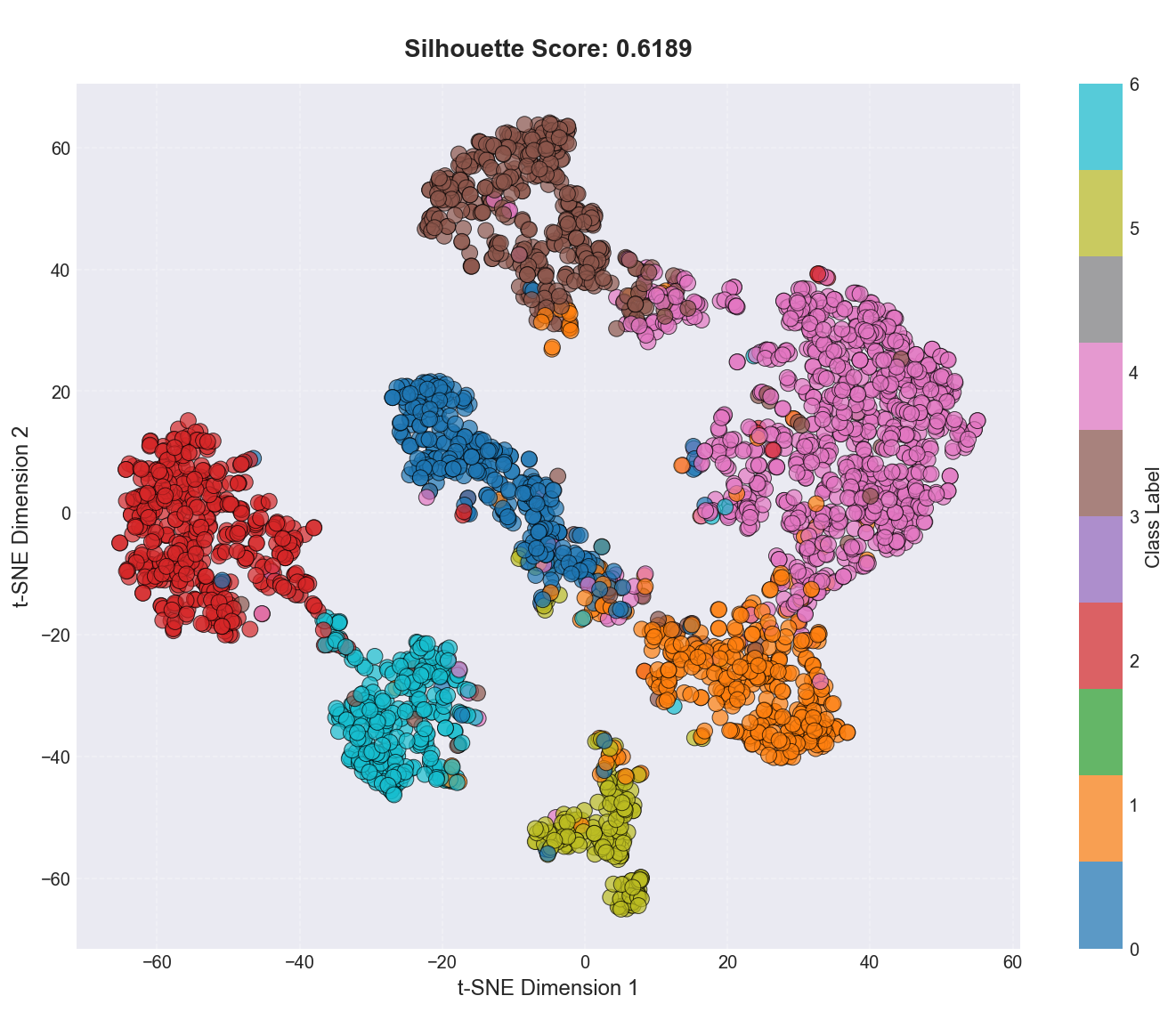} \\
\textbf{(a) Degree} & \textbf{(b) Core} \\[1em]
\includegraphics[width=0.48\textwidth]{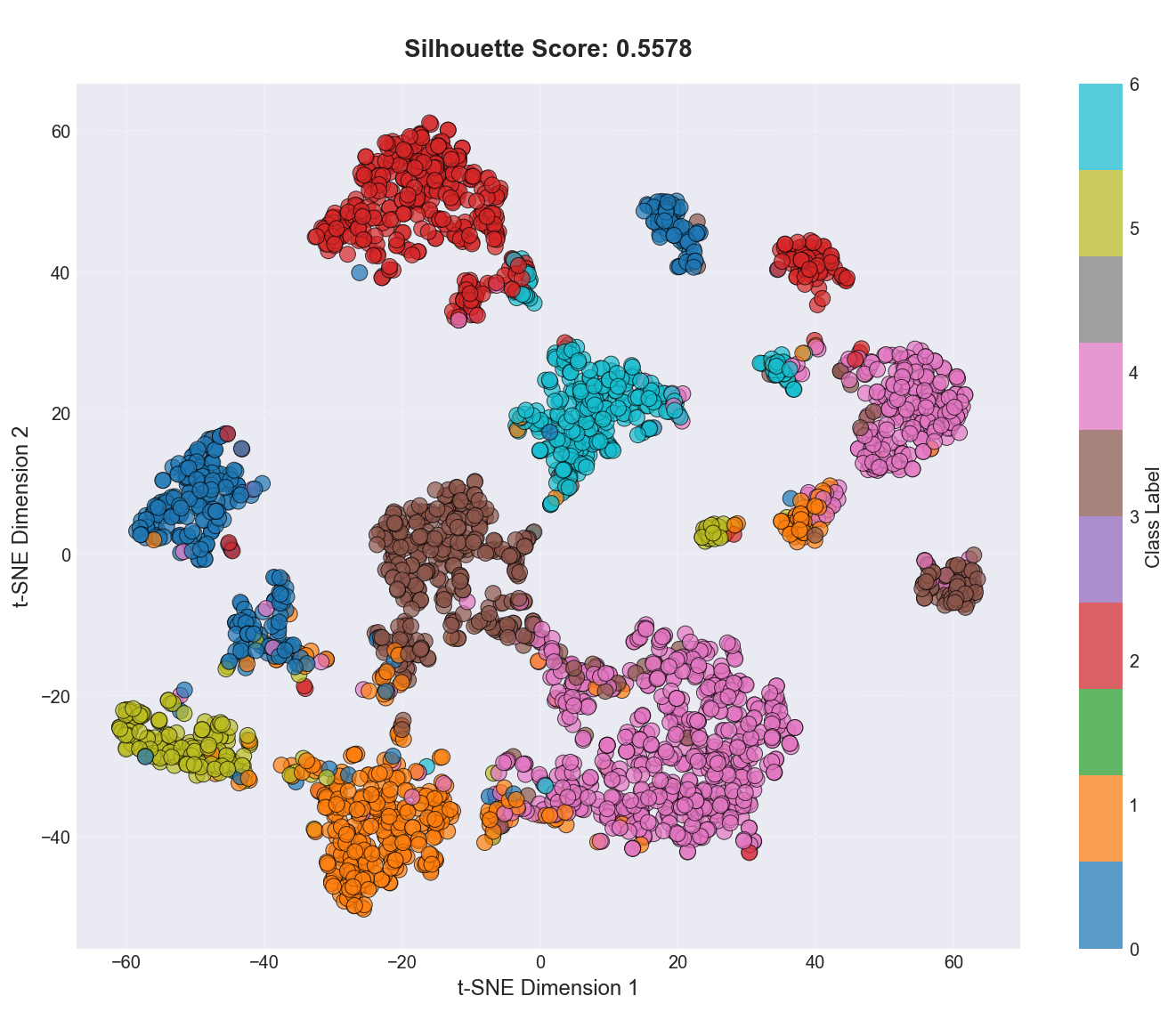} &
\includegraphics[width=0.48\textwidth]{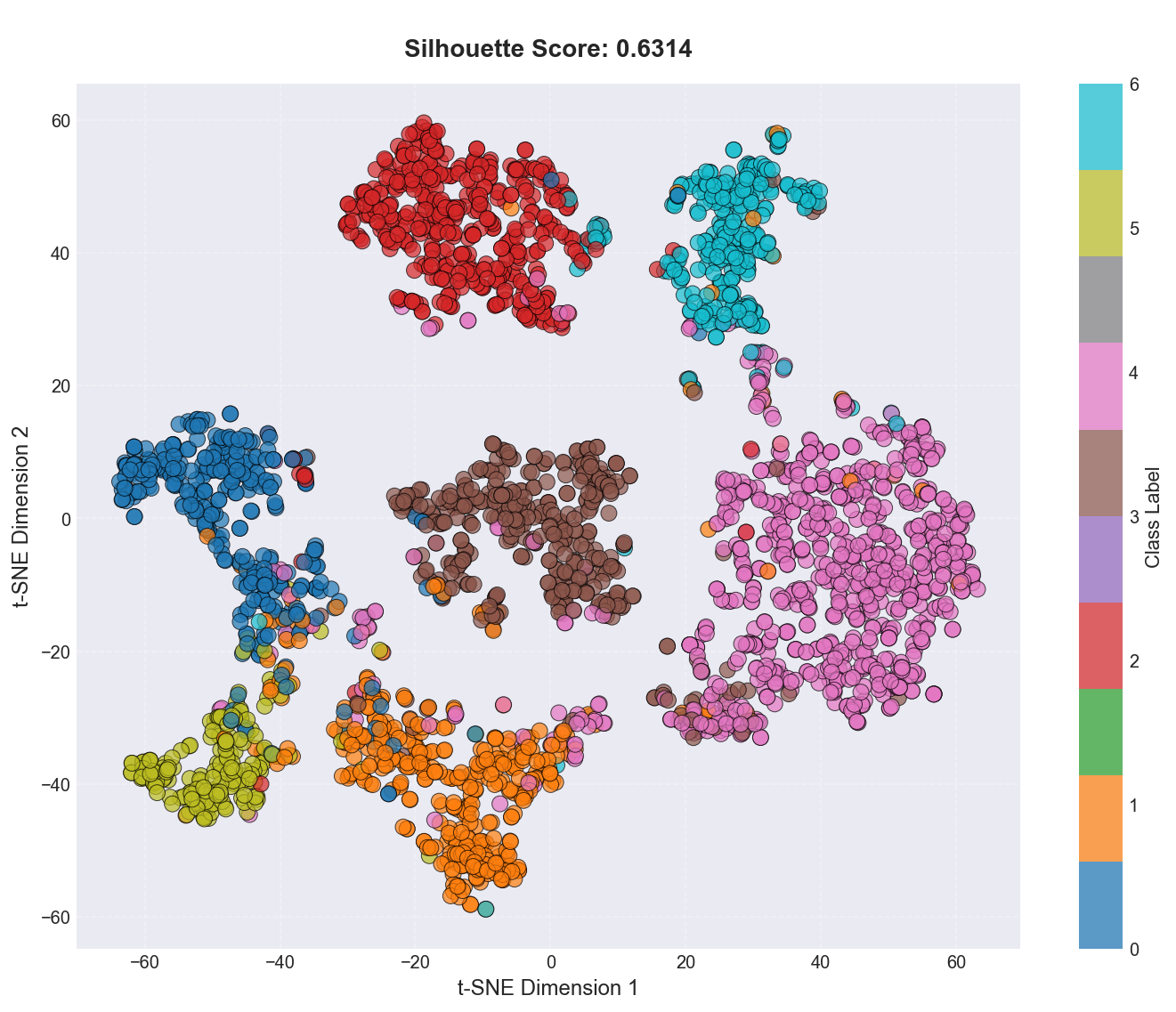} \\
\textbf{(c) Onion} & \textbf{(d) Learnable}
\end{tabular}
\caption{\textbf{t-SNE visualization of ISP-GNN embeddings on Cora-ML.} 
The learnable invariant produces the most compact and well-separated 
clusters (silhouette: 0.631), outperforming predefined invariants 
(silhouette: 0.558--0.619). This demonstrates that learnable stratification 
automatically discovers task-adaptive structural patterns aligned with 
semantic categories.}
\label{fig:tsne_comparison}
\end{figure*}

\begin{figure*}[t]
    \centering
    \includegraphics[width=\textwidth]{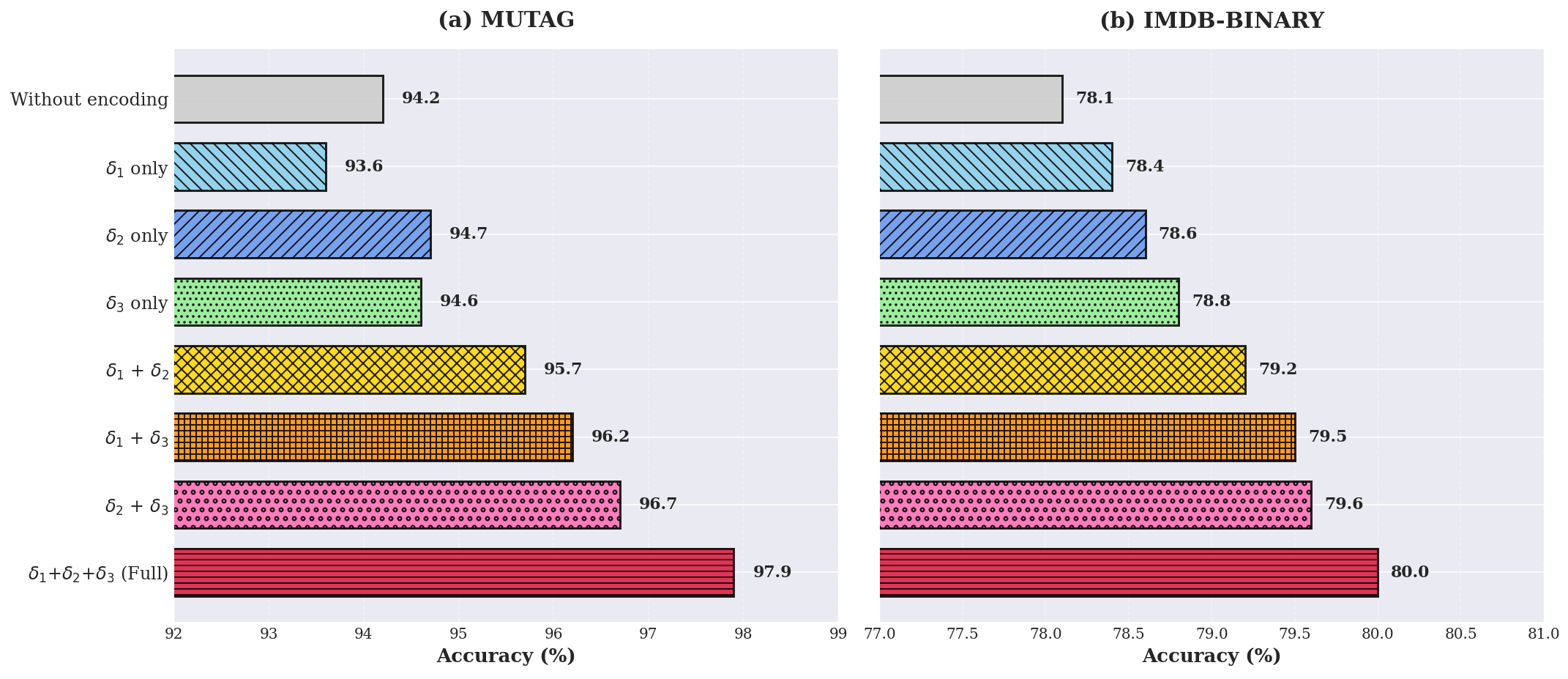}
    \caption{Ablation study on hierarchical gap encoding components for graph classification. Results demonstrate the complementary nature of $\delta_1$, $\delta_2$, and $\delta_3$. Progressive performance improvements from individual components through pairs to full encoding validate the necessity of all three gap features.}
    \label{fig:gap_ablation}
\end{figure*}

\subsection{Hierarchical Structural Gap Encoding}

In this subsection, we empirically justify the design choices related to our hierarchical structural gap encoding mechanism.

\subsubsection{Impact of the Motif Type}

\begin{table}[H]
\centering
\caption{Ablation study on motif types. Results show mean accuracy $\pm$ standard deviation (\%) for graph classification. Best results are highlighted in \textbf{bold}. Note that the degree is used as the invariant.}
\label{tab:motif_ablation}
\begin{tabular}{lccc}
\toprule
\textbf{Motif Type} & \textbf{MUTAG} & \textbf{PTC\_MR} & \textbf{IMDB-MULTI} \\
\midrule
GIN (Baseline) & $92.8\pm5.9$ & $65.6\pm6.5$ & $54.6\pm3.0$ \\
\midrule
4-clique & $95.2\pm4.4$ & $72.7\pm5.6$ & $55.2\pm2.8$ \\
4-cycle & $94.7\pm4.0$ & $71.0\pm7.5$ & $55.0\pm3.2$ \\
Diamond & $94.6\pm4.1$ & $73.4\pm7.7$ & $55.4\pm2.9$ \\
5-clique & $95.3\pm5.5$ & $72.2\pm7.7$ & $55.3\pm3.0$ \\
6-cycle & $94.7\pm5.0$ & $74.3\pm6.6$ & $54.9\pm3.1$ \\
\midrule
\textbf{Triangle} & $\mathbf{97.9\pm2.6}$ & $\mathbf{75.8\pm5.9}$ & $\mathbf{55.9\pm3.1}$ \\
\bottomrule
\end{tabular}
\end{table}

\begin{figure*}[t]
\centering
\scalebox{1.0}{
\begin{tabular}{cc}
\includegraphics[width=0.48\textwidth]{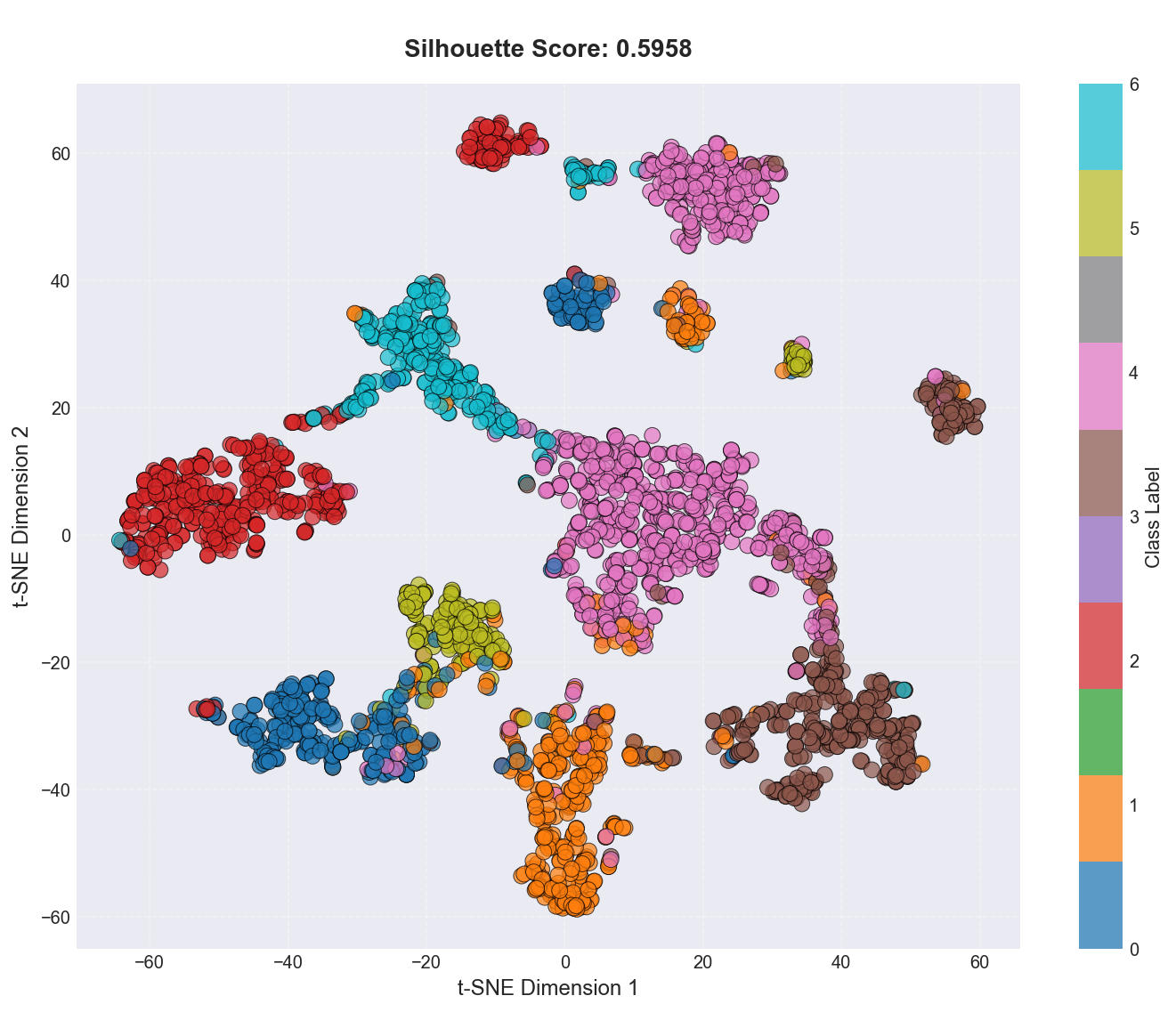} &
\includegraphics[width=0.48\textwidth]{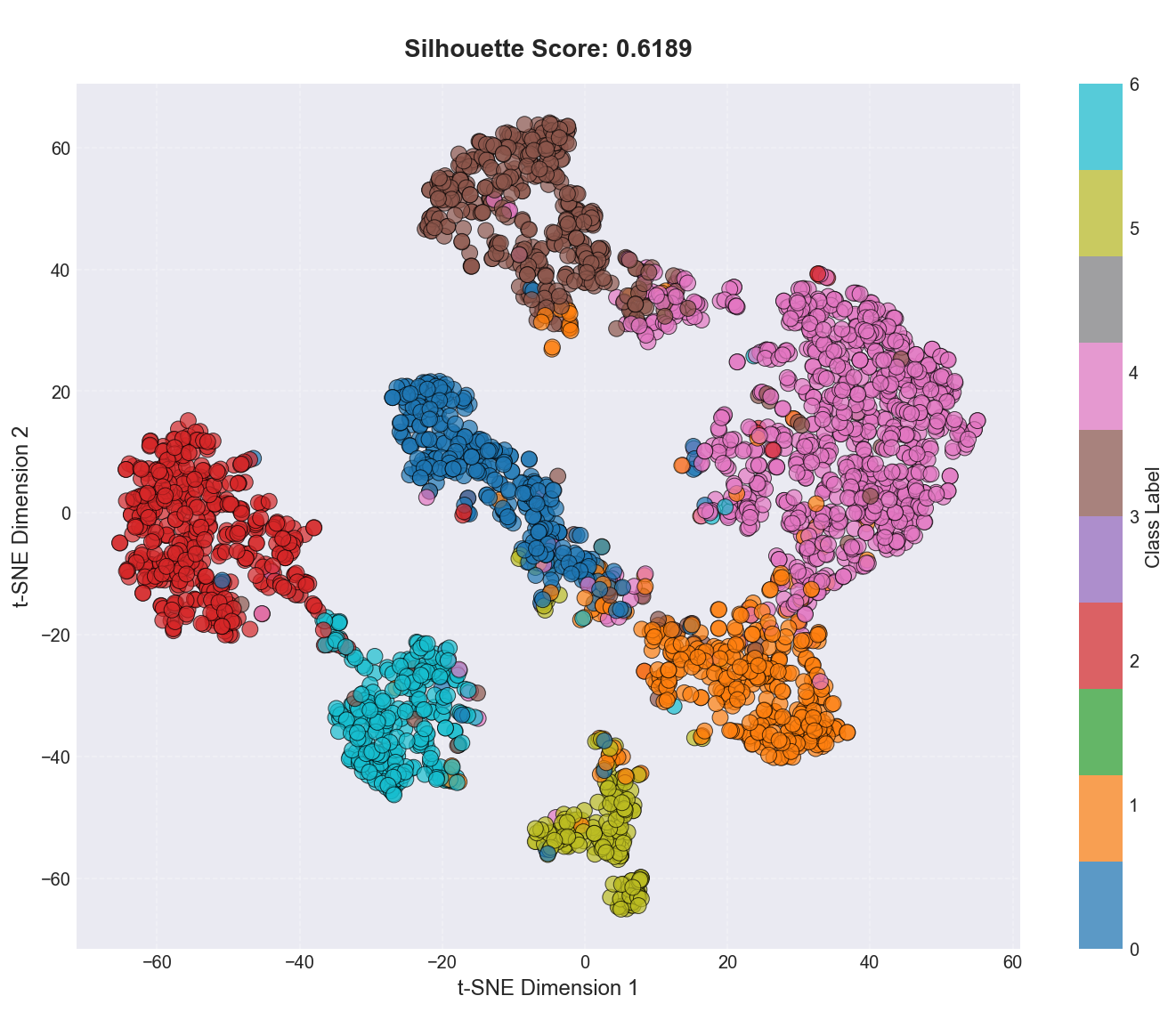} \\
\textbf{(a) Degree} & \textbf{(b) Core} \\[1em]
\includegraphics[width=0.48\textwidth]{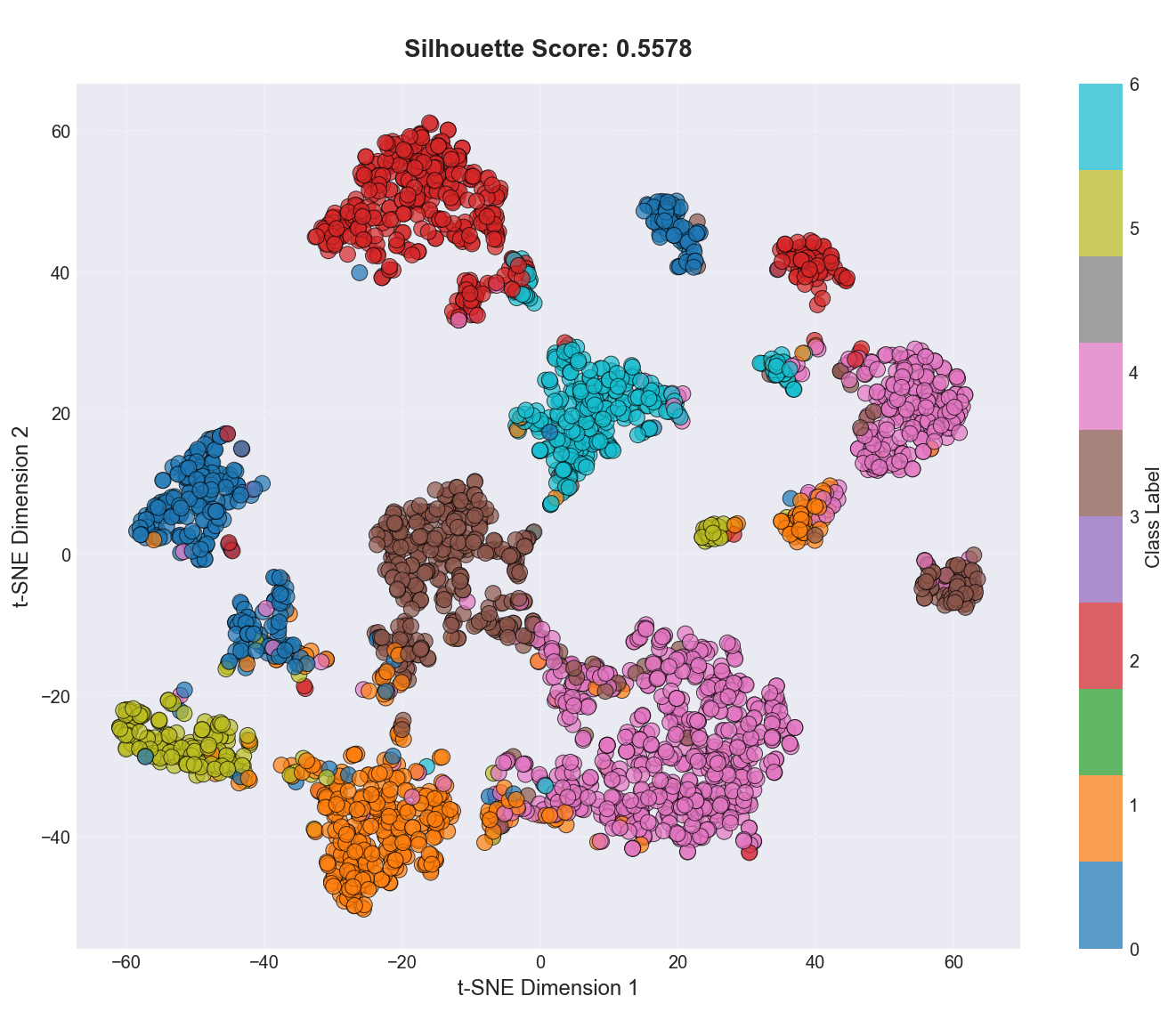} &
\includegraphics[width=0.48\textwidth]{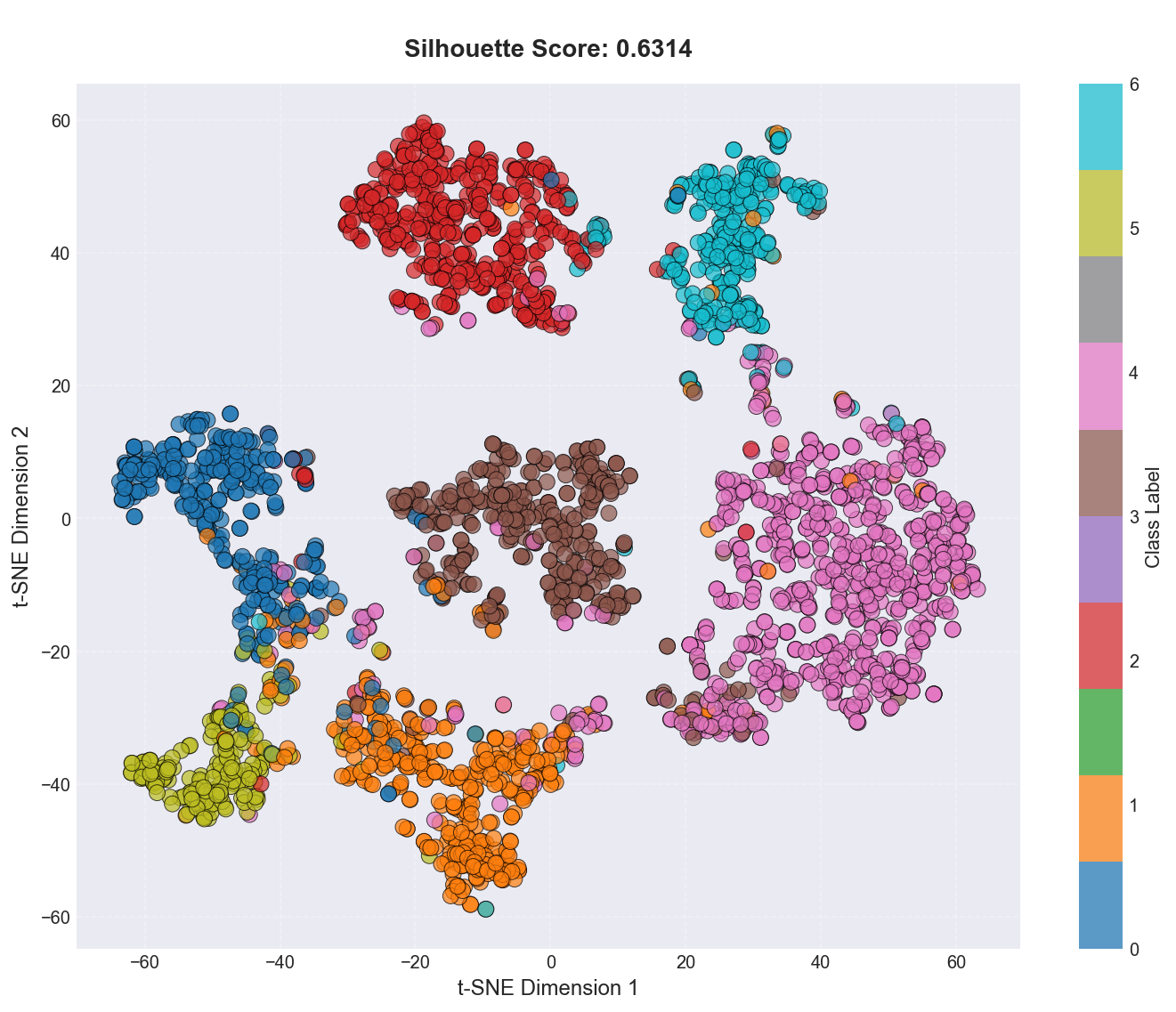} \\
\textbf{(c) Onion} & \textbf{(d) Learnable}
\end{tabular}
}
\caption{\textbf{t-SNE visualization of ISP-GNN embeddings on Cora-ML.} 
The learnable invariant produces the most compact and well-separated 
clusters (silhouette: 0.631), outperforming predefined invariants 
(silhouette: 0.558--0.619). This demonstrates that learnable stratification 
automatically discovers task-adaptive structural patterns aligned with 
semantic categories.}
\label{fig:tsne_comparison}
\end{figure*}

Table~\ref{tab:motif_ablation} validates our choice of triangles through ablation studies with five motif types. Triangle-based aggregation consistently outperforms higher-order alternatives (4-cliques, 5-cliques, 4-cycles, diamonds, 6-cycles) across all datasets. This superiority arises from three complementary factors. First, motif abundance: triangles are more prevalent than other higher-order motifs, ensuring robust learning signals across the entire graph. Second, information completeness: triangles serve as fundamental building blocks of higher-order structures, as larger cliques decompose into overlapping triangles, allowing triangle-based aggregation to capture essential structural information without explicitly enumerating rarer motifs. Third, fewer parameters: triangles use fewer parameters while providing substantially more training instances, thereby reducing the risk of overfitting. Overall, triangles provide an optimal balance of expressivity, coverage, and efficiency for the ISP framework.

\subsubsection{Impact of different Components in Hierarchical Gap Encoding}

To validate the contribution of each component in our hierarchical gap encoding, we conduct ablation studies by systematically removing individual gap features and their combinations. Figure~\ref{fig:gap_ablation} presents results across MUTAG and IMDB-BINARY datasets.

The results reveal three key findings. First, all individual gap components improve over the baseline without encoding, with $\delta_2$ and $\delta_3$ providing stronger contributions than $\delta_1$ across both datasets, suggesting that minimum gaps and neighbor homogeneity carry more discriminative power than maximum gaps alone. Second, pairwise combinations consistently outperform individual components, demonstrating complementarity among gap features.  Third, the full encoding $\delta_1+\delta_2+\delta_3$ achieves the best performance, confirming that combining all three components contributes unique and synergetic structural information.

\clearpage

\subsection{Learnable Invariant Analysis} \label{sec:learnable_inv}

In this section, we provide an analysis of the capabilities and hyperparameter sensitivity of the ISP-GNN$_\dagger$ variant.

\subsubsection{Task-Adaptive Structural Discovery}

To validate the learnable invariant of the ISP-GNN capture task-relevant structural patterns, 
We analyse embedding quality via t-SNE projection \cite{maaten2008visualizing} and silhouette scores \cite{shahapure2020cluster} on the Cora-ML dataset for the node classification task. 
Figure~\ref{fig:tsne_comparison} shows that ISP-GNN with learnable invariant 
produces the most compact, well-separated clusters with the highest silhouette score, 
outperforming all predefined variants.

Unlike fixed  invariants that capture only specific structural properties (degree: local 
connectivity; core: cohesiveness; onion: hierarchy), the learnable variant 
 adapts its stratification to align with semantic categories in the groundtruth, learning 
an optimal combination of structural signals for each task.

\begin{figure*}[t]
\centering
\includegraphics[width=0.95\textwidth]{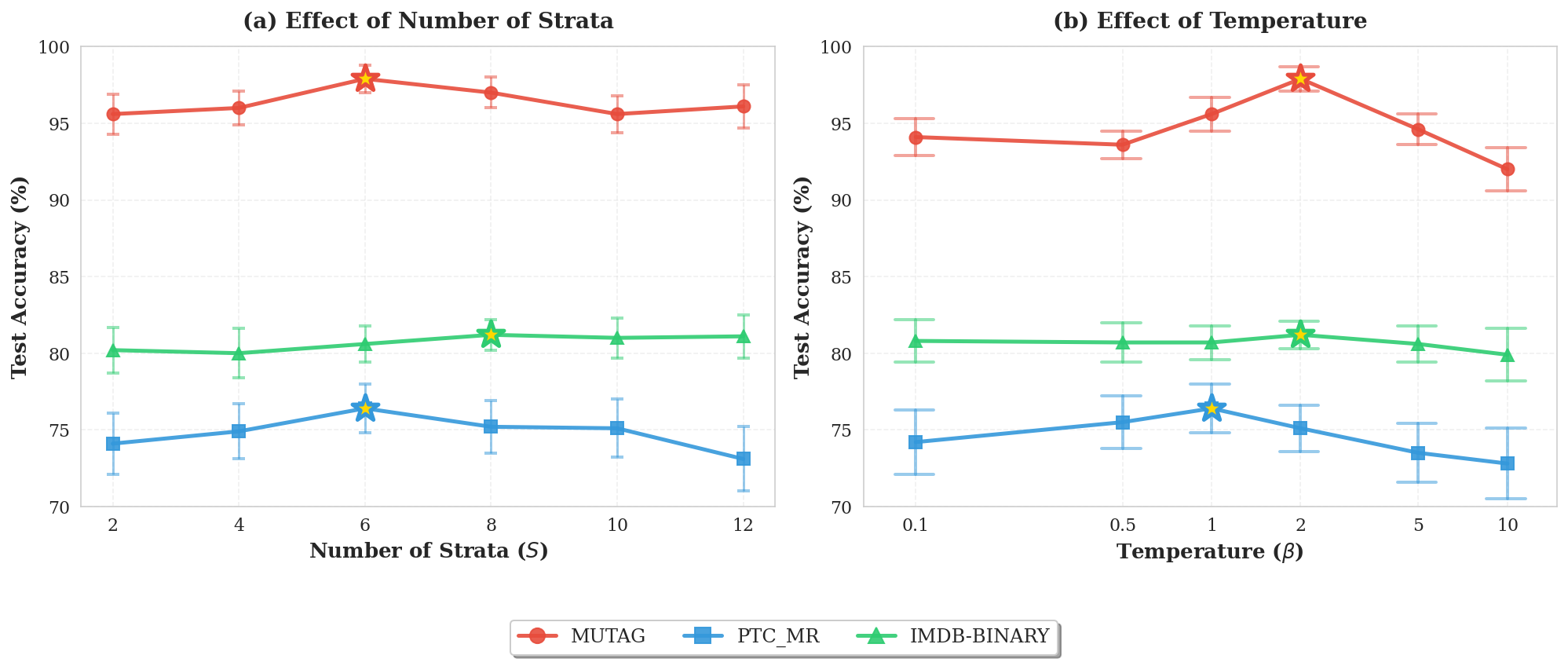}
\caption{\textbf{Hyperparameter sensitivity analysis.} (a) Effect of number of strata $S$ (b) Effect of temperature $\beta$.}
\label{fig:hyperparam}
\end{figure*}

\subsubsection{Hyperparameter Sensitivity Analysis}
\label{sec:hyperparameter_analysis}

We analyze the sensitivity of ISP-GNN$_\dagger$ to two key hyperparameters: the number of strata $S$ and the soft assignment temperature $\beta$. Figure~\ref{fig:hyperparam} presents results on the graph classification task.

ISP-GNN$^{\dagger}$ exhibits robust performance with $S \in [6, 8]$ and $\beta \in [0.5, 2.0]$, providing practical guidance for hyperparameter selection. While optimal values are dataset-dependent, performance remains stable across these ranges, with lower variation. Notably, $S \in [6, 8]$ indicates that a relatively small number of strata suffices for optimal performance, which is beneficial from both computational efficiency and memory footprint perspectives.

\subsection{Scalability Analysis}
\label{sec:scalability}

We empirically validate the theoretical complexity bounds (Theorems~\ref{thm:convergence} and~\ref{thm:complexity}) of ISP-WL through experiments on synthetic Barabási-Albert (BA) graphs ranging from 10K to 1M nodes with average degree 10. BA graphs are scale-free networks that represent real-world topologies, including social networks, citation networks, and biological networks~\cite{barabasi1999emergence}. These graphs maintain bounded degeneracy, a property commonly found in real-world sparse graphs~\cite{demaine2019structural}.

In this experiment, we measure per-iteration time for ISP-WL variants (Degree, K-Core, Onion, K-Truss) across six graph sizes (10K, 50K, 100K, 250K, 500K, 1M nodes). All experiments use 5 WL iterations with preprocessing costs amortized over iterations to isolate per-iteration overhead.

\subsubsection{Validation of Theorem~\ref{thm:complexity}.}

Figure~\ref{fig:isp-wl-variants}(a) demonstrates linear per-iteration scaling consistent with $O(|V| + |E|)$ complexity: Degree, K-Core, and Onion variants maintain near-linear growth across three orders of magnitude. K-Truss exhibits slightly superlinear scaling due to its $O(|E|^{1.5})$ preprocessing, but remains practical even at 1M nodes.

The overhead ratios in Figure~\ref{fig:isp-wl-variants}(b) reveal an important scaling behavior: initial spikes at medium graphs (50K--100K nodes) occur when preprocessing dominates the small baseline per-iteration cost, but this overhead then decreases and stabilizes at large scale (500K--1M nodes) as preprocessing becomes better amortized. This validates that ISP-WL's per-iteration complexity matches standard 1-WL asymptotically.

\subsubsection{Preprocessing Costs.}

Figure~\ref{fig:preprocessing}(a) isolates preprocessing complexity: Degree computation is near-instantaneous, K-Core and Onion scale linearly with edges, while K-Truss's triangle enumeration per edge yields superlinear cost. However, when amortized over iterations (Figure~\ref{fig:preprocessing}(b)), even K-Truss preprocessing becomes manageable.

\subsubsection{Triangle Enumeration.}
Figure~\ref{fig:triangles} validates the complexity $T = O(d|E|)$ assumption for sparse graphs. Triangle count grows sublinearly (panel a), enumeration time scales linearly (panel b), and panel (c) demonstrates the quadratic explosion at high degree that makes bounded degeneracy essential for practical efficiency.

\subsubsection{Validation of Theorem~\ref{thm:convergence}.}
All ISP-WL variants converged within 5 iterations across all graph sizes, well below the $K \leq |V|$ bound. The stable overhead ratios at scale confirm that convergence time remains independent of graph size.

Overall, these results confirm that ISP-WL maintains efficient complexity on sparse scale-free graphs with stable overhead depending on invariant choice. Degree offers the lowest overhead, K-Core/Onion balance efficiencywith expressivity, while K-Truss (5.5×) provides maximum structural distinguishability at moderate cost when preprocessing is amortized over iterations. These results generalize to real-world sparse networks with similar bounded degeneracy properties.

\begin{figure*}[t]
    \centering
    \includegraphics[width=0.9\linewidth]{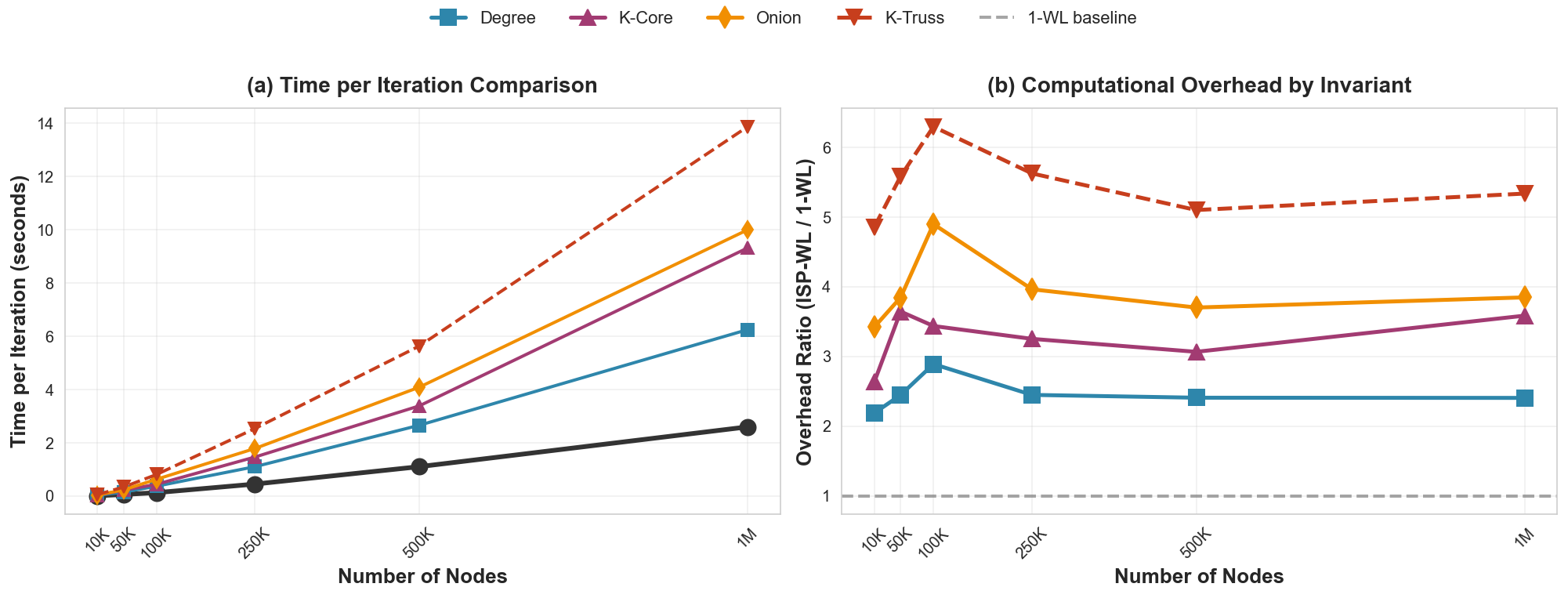}
    \caption{ISP-WL computational overhead on synthetic BA graphs. \textbf{(a)} Time per iteration scales linearly with graph size for all variants. \textbf{(b)} Overhead ratios stabilize at large graphs, with preprocessing costs becoming better amortized at scale.}
    \label{fig:isp-wl-variants}
\end{figure*}

\begin{figure*}[t]
    \centering
    \includegraphics[width=0.9\linewidth]{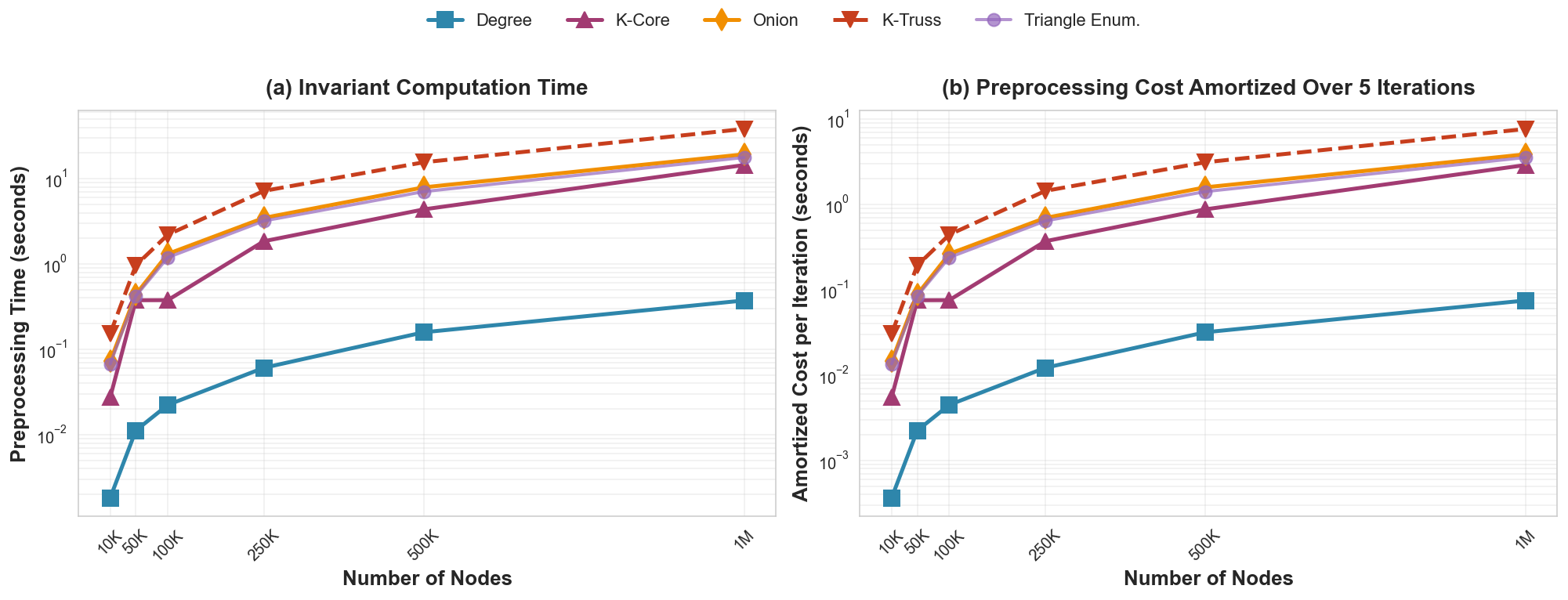}
    \caption{Preprocessing cost analysis (log scale). \textbf{(a)} Invariant computation time shows Degree ($O(|V|)$) is nearly instantaneous, K-Core/Onion ($O(|E|)$) remain efficient, while K-Truss ($O(|E|^{1.5})$) is expensive but amortizes well at scale. \textbf{(b)} Amortized cost per iteration shows preprocessing contribution decreases relative to the iteration cost as the graphs grow.}
    \label{fig:preprocessing}
\end{figure*}

\begin{figure*}[t]
    \centering
    \includegraphics[width=0.9\linewidth]{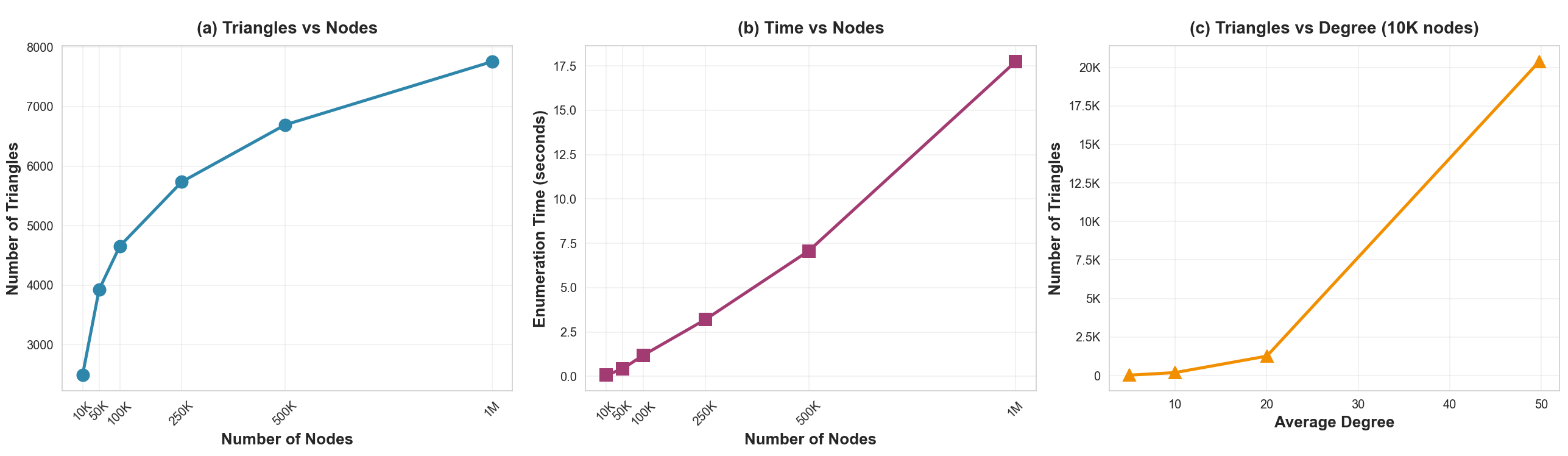}
    \caption{Triangle enumeration scalability on BA graphs. \textbf{(a)} Sublinear triangle growth confirms bounded degeneracy typical of real-world networks. \textbf{(b)} Linear enumeration time validates the complexity $T = O(d|E|)$. \textbf{(c)} Quadratic explosion at high degrees demonstrates why bounded degeneracy is critical.}
    \label{fig:triangles}
\end{figure*}

\clearpage

\begin{figure*}[t]
    \centering
    \includegraphics[width=\textwidth]{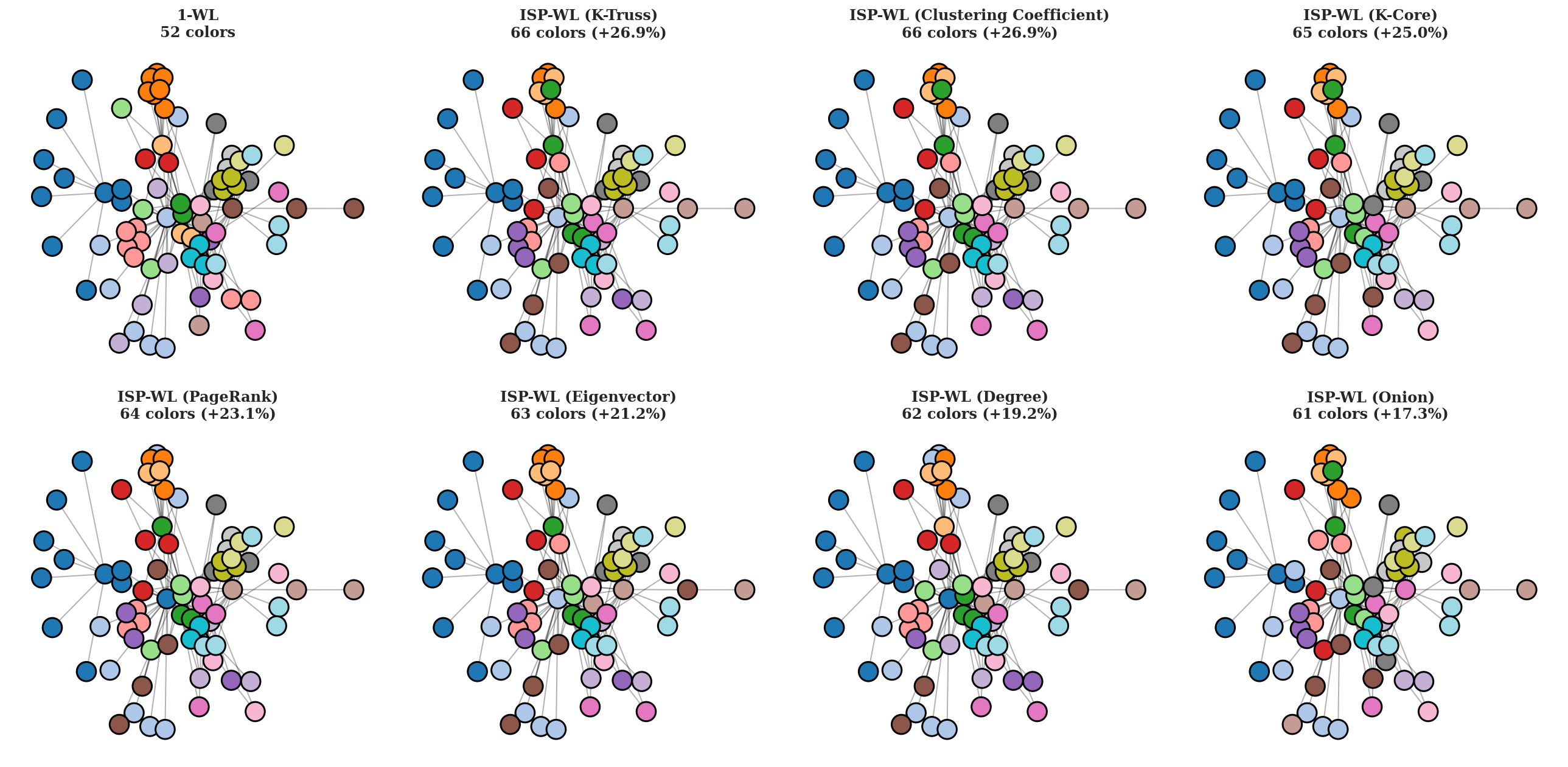}
    \caption{Coloring refinement comparison of 1-WL and ISP-WL. ISP-WL with different invariants achieves 17-27\% more colours than standard 1-WL, providing higher discriminative power.}
    \label{fig:coloring_visualization}
\end{figure*}

\subsection{Visualization Analysis} \label{sec:visualization}

In this section, we present various visualizations of ISP-WL. We employ the Les Misérables co-occurrence network, which consists of 77 nodes and 254 edges.

\subsubsection{Coloring Refinement Across Invariants}

Figure~\ref{fig:coloring_visualization} illustrates the discriminative power of ISP-WL across seven graph invariants. While standard 1-WL achieves 52 distinct colorings, all ISP-WL variants surpass this baseline, with improvements ranging from moderate to substantial depending on the invariant choice. This improvement improves node distinguishability, validating Theorem \ref{thm:nodedist}. Powerful invariants (K-Truss, Clustering Coefficient) achieve the strongest refinement, demonstrating that measures aligned with ISP-WL's theoretical insights (Theorem \ref{thm:invariant-expressivity}). Decomposition-based measures (K-Core) follow closely, effectively capturing the network's core-periphery structure. Global centrality measures (PageRank, Eigenvector) provide moderate improvements by identifying influential nodes that 1-WL cannot distinguish. Simpler local measures (Degree, Onion) show the weakest, though still consistent, gains, as their structural information partially overlaps with 1-WL's neighbourhood aggregation.

Visually, the increased colour diversity in ISP-WL variants reveals finer-grained structural distinctions, particularly in the densely connected central region where characters have similar local neighbourhoods but differ in higher-order connectivity patterns. The variation in performance across invariants validates our framework's key principle: expressivity gains depend critically on choosing invariants that capture structural properties orthogonal to standard message passing.

\subsubsection{Invariant Orthogonality Analysis}

Figure~\ref{fig:complementarity} quantifies structural complementarity between invariants through pairwise coloring differences. Most invariant pairs show high divergence, indicating each captures largely distinct structural properties. This orthogonality validates ISP-WL's sensitivity to invariant choice, where different structural lenses produce fundamentally different node distinctions.

Notable patterns emerge: K-Truss and clustering coefficient exhibit the lowest divergence, revealing their shared emphasis on triangular cohesion. Degree shows the most consistent divergence from all other invariants, confirming it captures only first-order connectivity with no higher-order information. The predominantly red matrix confirms that carefully chosen invariants provide orthogonal structural views that ISP-WL successfully exploits for expressivity gains.

\begin{figure*}[t]
    \centering
    \includegraphics[width=0.5\textwidth]{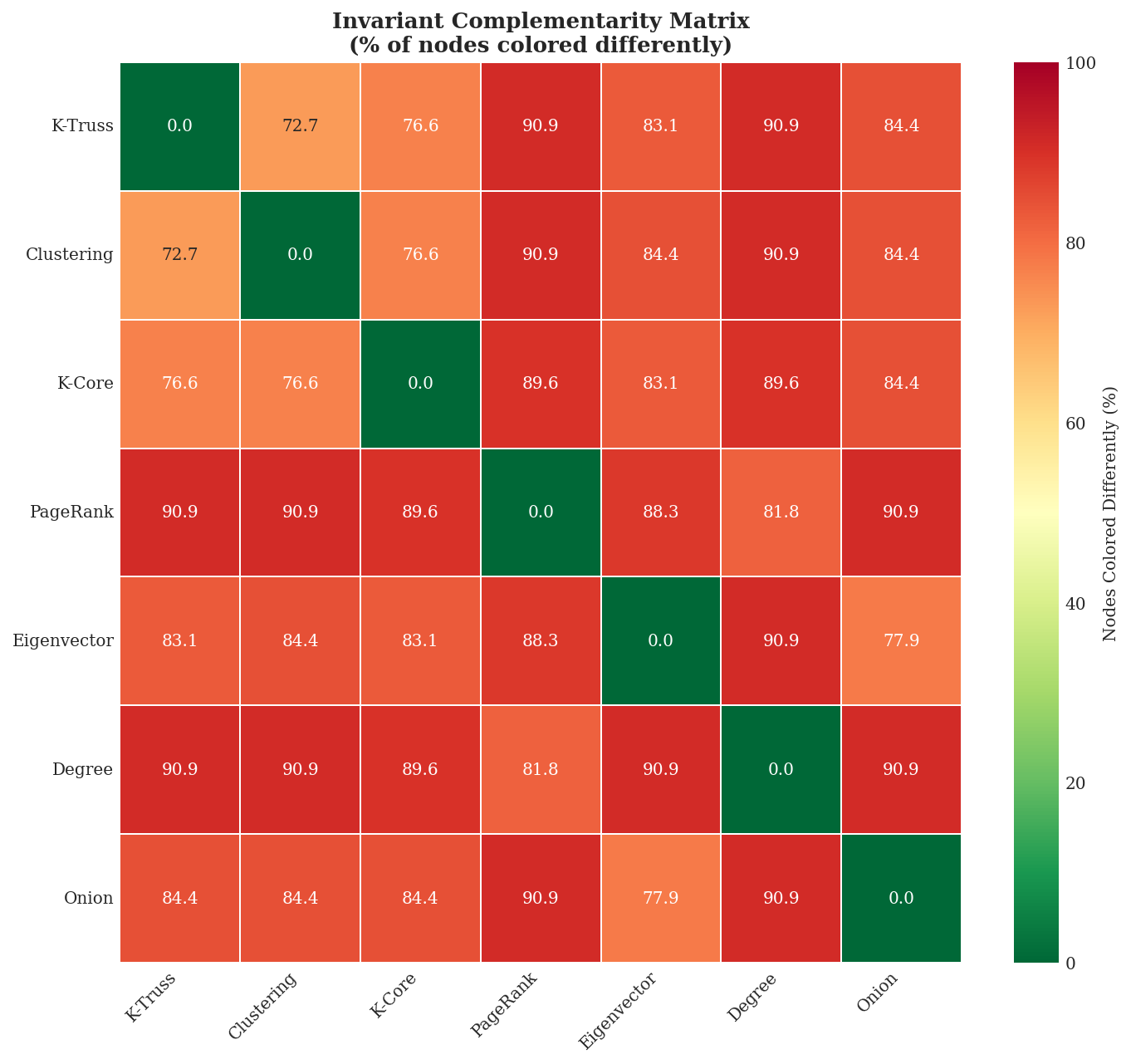}
    \caption{Pairwise complementarity between invariants measured as percentage of nodes receiving different final colors. High values (red) indicate orthogonal structural information, validating that different invariants capture distinct graph properties.}
    \label{fig:complementarity}
\end{figure*}

\subsubsection{Convergence Dynamics}

Figure~\ref{fig:color_evolution} reveals distinct convergence patterns across invariants. Standard 1-WL converges rapidly at iteration 3, establishing a baseline of 52 colors. ISP-WL variants exhibit differential convergence speeds: K-Truss achieves early stabilization at iteration 3 with 66 colors, while clustering coefficient, K-Core, PageRank, and Eigenvector require more iterations to converge.

K-Truss's rapid convergence stems from its coarse stratification—the graph contains limited truss structures, enabling quick layer-by-layer processing. In contrast, finer-grained invariants like Clustering and PageRank create richer initial stratifications that propagate structural distinctions over additional iterations. 

\begin{figure*}[]
    \centering
    \includegraphics[width=\textwidth]{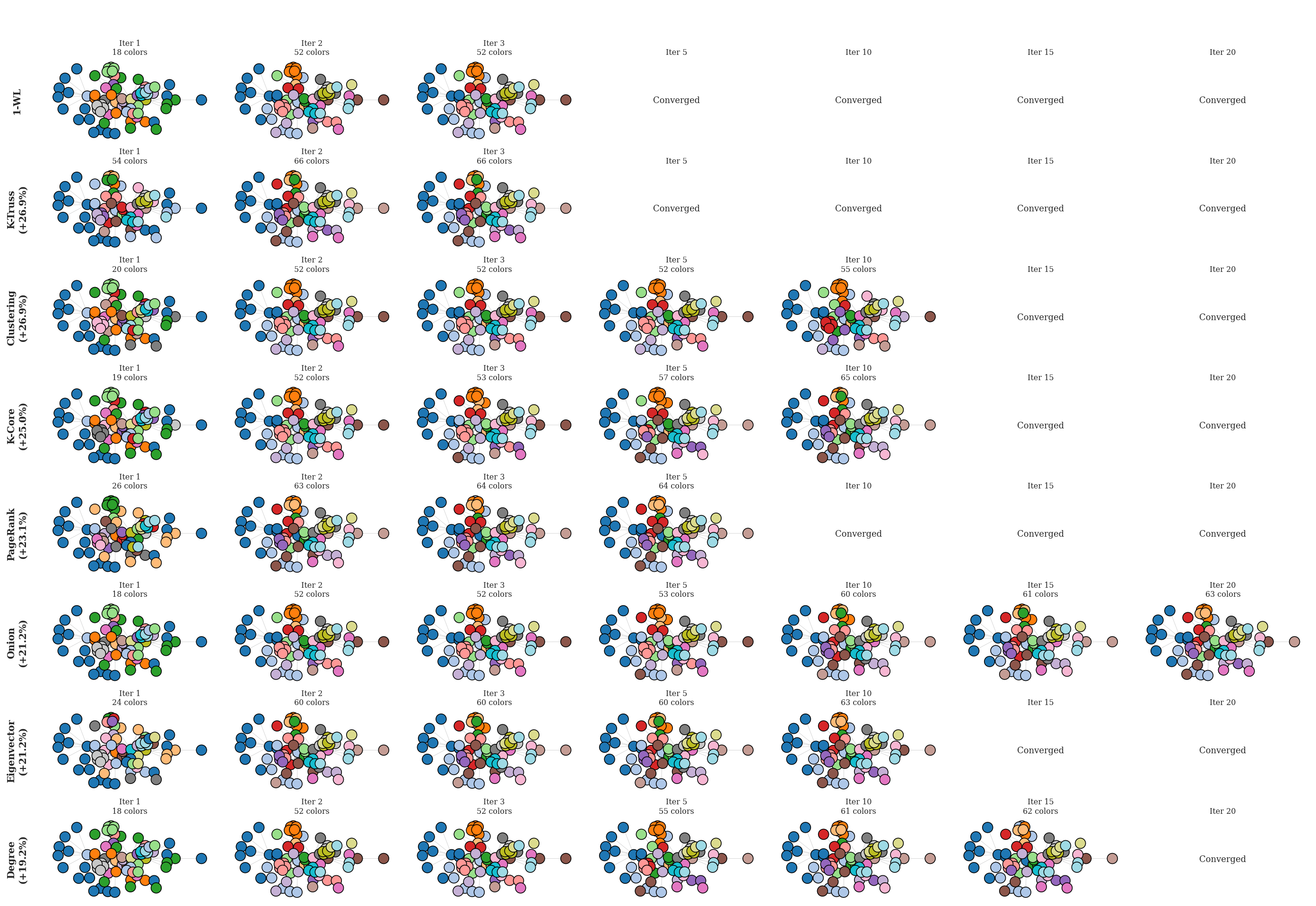}
    \caption{Iterative color refinement across variants. K-Truss converges fastest, while other finer-grained invariants, such as degree and onion, require more iterations.}
    \label{fig:color_evolution}
\end{figure*}

\clearpage

\section{Limitations and Future Work} \label{sec:limitations}

While our ISP framework offers a principled approach to enhancing GNN expressivity through invariant-stratified propagation, our work has several limitations that present opportunities for future research.

\subsection{Computational Complexity for Dense Graphs}

Our theoretical complexity analysis (Theorem \ref{thm:complexity}) establishes that ISP-WL achieves $\mathcal{O}(K(|V| + |E|))$ complexity on sparse graphs with bounded degeneracy. However, this relies on the assumption that the number of triangles $T = \mathcal{O}(d|E|)$ where $d$ is constant. While this holds for many real-world networks, including social networks, molecular graphs, and citation networks, it may not hold for dense graphs or graphs with high clustering coefficients. Triangle enumeration can become prohibitive for nodes with very high degrees, as shown in Figure \ref{fig:triangles}(c). Future work could explore approximate triangle counting methods or sampling-based approaches to maintain efficiency on denser graph structures while preserving the essential structural information captured by ISP.

\subsection{Triangle-Based Aggregation}

Our current implementation primarily focuses on triangle-based higher-order aggregation due to its computational efficiency and abundance in real-world graphs. While Section 9.2 presents a theoretical extension to arbitrary $k$-motifs and Table 12 provides empirical evaluation across multiple motif types, our experiments demonstrate that triangles consistently outperform these alternatives on the tested benchmarks. However, this superiority may not generalize to all graph types. Specifically, graphs with limited triangle structure (e.g., bipartite networks, tree-like structures, or sparse biological networks) may benefit more from alternative motif patterns. Future work should investigate adaptive motif selection mechanisms that can automatically identify and leverage the most informative higher-order patterns for specific graph families, extending our preliminary comparative analysis to a broader range of graph structures and application domains.

\subsection{Hierarchical Gap Encoding Design}

The three-component gap encoding $\delta(v, u, w) = (\delta_1, \delta_2, \delta_3)$ presented in Definition 3.1 is designed based on intuitions about structural heterogeneity. While our ablation studies (Figure 10) demonstrate that all three components contribute complementary information, the choice of max/min/inter-neighbor gaps is not formally proven to be minimal or optimal. Alternative gap functions or learned gap encodings might capture structural heterogeneity more effectively for specific domains. Extending our framework to support learnable gap functions that can adapt to task-specific patterns would enhance flexibility and potentially improve performance.

\subsection{Expressivity Characterization}

Theorem 4.3 establishes that ISP-WL's expressivity is invariant-dependent, but this characterization remains somewhat incomplete. While we prove that $\text{1-WL} \prec \text{ISP-WL}(\phi) \prec \text{3-WL}$ when $\phi \preceq \text{1-WL}$, the exact position of ISP-WL in the Weisfeiler-Leman hierarchy depends on the chosen invariant. A more precise characterisation of which graph pairs ISP-WL can and cannot distinguish relative to higher-order WL methods would strengthen the theoretical foundations. Future work would identify specific graph families where ISP-WL matches or exceeds $k$-WL for $k > 1$, and characterize the structural patterns that remain beyond ISP-WL's distinguishing power.

\subsection{Invariant Selection}

While our framework supports a variety of graph invariants and shows that different invariants excel in different tasks (as illustrated in Table 11), the selection of invariants currently requires manual intervention and relies on domain knowledge. In our future work, we aim to develop a more systematic, data-driven method for selecting the most appropriate invariants. 

\subsection{Scalability to Massive Graphs}

Although our scalability experiments demonstrate linear scaling up to 1M nodes on synthetic BA graphs (Figure 13), real-world graph learning increasingly involves graphs with tens or hundreds of millions of nodes. At such scales, even $\mathcal{O}(d|E|)$ triangle enumeration may become challenging despite theoretical efficiency. Distributed computing approaches, mini-batch training strategies that approximate global invariant rankings, or sampling-based ISP mechanisms could extend our framework to massive graphs. Additionally, investigating the trade-off between approximate triangle counting and expressivity would enable practitioners to balance computational constraints with model performance.